\PassOptionsToPackage{shortlabels}{enumitem}
\documentclass[11pt, a4paper]{include/gdm_format}

\usepackage{amsmath,amsfonts,bm}

\newcommand{\params}{{\theta}}

\newcommand{\abias}{{\tA}}

\newcommand{\norm}[1]{\left\lVert #1 \right\rVert}
\newcommand{\bnorm}[1]{\big\lVert #1 \big\rVert}
\newcommand{\Bnorm}[1]{\Big\lVert #1 \Big\rVert}

\def\eqref#1{equation~\ref{#1}}

\def\1{\bm{1}}

\DeclareMathAlphabet{\mathsfit}{\encodingdefault}{\sfdefault}{m}{sl}
\SetMathAlphabet{\mathsfit}{bold}{\encodingdefault}{\sfdefault}{bx}{n}
\newcommand{\tens}[1]{\bm{\mathsfit{#1}}}
\def\tA{{\tens{A}}}

\def\gL{{\mathcal{L}}}

\def\gO{{\mathcal{O}}}

\def\sR{{\mathbb{R}}}

\newcommand{\E}{\mathbb{E}}

\newcommand{\R}{\mathbb{R}}

\newcommand{\softmax}{\mathrm{softmax}}

\newcommand{\Var}{\mathrm{Var}}

\usepackage[authoryear, sort&compress, round]{natbib}
\usepackage{xcolor}
\usepackage{subcaption}
\usepackage{mwe}
\usepackage{graphicx}
\usepackage{xr}
\usepackage{marvosym}
\usepackage{makecell}
\usepackage[table]{xcolor}
\usepackage{float}

\usepackage{fvextra}

\definecolor{bgcolor}{rgb}{0.95,0.95,0.95}

\definecolor{lblA}{RGB}{33,113,181}
\definecolor{lblB}{RGB}{230,85,13}

\usepackage{tablefootnote}
\usepackage{pdfpages}
\usepackage{etoc}

\usepackage{hyperref}
\usepackage{url}
\usepackage{xurl}
\usepackage[parfill]{parskip}

\usepackage{amsmath}
\usepackage{xspace}
\usepackage{multirow}
\usepackage{paracol}
\usepackage{booktabs}
\usepackage{colortbl}
\usepackage{cleveref}
\usepackage{subcaption}
\usepackage{wrapfig}
\usepackage{soul}
\usepackage{multicol}
\usepackage{listings}
\usepackage{changes}
\usepackage{amsthm}
\usepackage{enumitem}
\usepackage{thmtools}
\usepackage{thm-restate}
\usepackage{enumitem}

\declaretheorem[name=Theorem,numberwithin=section]{theorem}
\declaretheorem[name=Lemma,sibling=theorem]{lemma}
\declaretheorem[name=Proposition,sibling=theorem]{proposition}

\declaretheorem[style=definition,name=Definition,sibling=theorem]{definition}

\declaretheorem[style=remark,name=Remark,sibling=theorem]{remark}
\declaretheorem[name=Observation]{observation}

\newcolumntype{P}[1]{>{\centering\arraybackslash}p{#1}}
\newcolumntype{M}[1]{>{\centering\arraybackslash}m{#1}}

\newcommand{\ouralgo}{DroPE\xspace}

\usepackage[framemethod=TikZ]{mdframed}
\mdfdefinestyle{theorem}{
    linecolor=white,
    outerlinewidth=1pt,
    roundcorner=5pt,
    innertopmargin=10pt,
    innerbottommargin=3pt,
    innerrightmargin=5pt,
    innerleftmargin=5pt,
    skipabove=\baselineskip,
    skipbelow=\baselineskip,
    backgroundcolor=black!6!white
}
    
\mdfdefinestyle{observation}{
    linecolor=white,
    outerlinewidth=1pt,
    roundcorner=5pt,
    innertopmargin=10pt,
    innerbottommargin=3pt,
    innerrightmargin=5pt,
    innerleftmargin=5pt,
    skipabove=\baselineskip,
    skipbelow=\baselineskip,
    backgroundcolor=blue!6!white
}

\title{Extending the Context of Pretrained LLMs by Dropping Their Positional Embeddings}

\correspondingauthor{Yoav Gelberg (\texttt{yoav@robots.ox.ac.uk}), Edoardo Cetin (\texttt{edo@sakana.ai})}

\author[1,2]{Yoav Gelberg}
\author[1]{Koshi Eguchi}
\author[1]{Takuya Akiba}
\author[1]{Edoardo Cetin}

\affil[1]{Sakana AI}
\affil[2]{University of Oxford}

\begin{document}

\maketitle

\begin{abstract}
    \vspace{-2.0em}
    \looseness=-1
    So far, expensive finetuning beyond the pretraining sequence length has been a requirement for effectively extending the context of language models (LM). In this work, we break this key bottleneck by \emph{\textbf{Dro}pping the \textbf{P}ositional \textbf{E}mbeddings of LMs after training (DroPE)}. Our simple method is motivated by three key theoretical and empirical observations. First, positional embeddings (PEs) serve a crucial role during pretraining, providing an important inductive bias that significantly facilitates convergence. Second, over-reliance on this explicit positional information is also precisely what prevents test-time generalization to sequences of unseen length, even when using popular PE-scaling methods. Third, positional embeddings are not an inherent requirement of effective language modeling and can be safely \emph{removed after pretraining} following a short recalibration phase. Empirically, DroPE yields seamless \emph{zero-shot} context extension \emph{without any long-context finetuning}, quickly adapting pretrained LMs without compromising their capabilities in the original training context. Our findings hold across different models and dataset sizes, far outperforming previous specialized architectures and established rotary positional embedding scaling methods.
    \vspace{0.5em}
    \begin{center}
        \begin{tabular}{rcl}
            \raisebox{-1.5pt}{\includegraphics[height=1.05em]{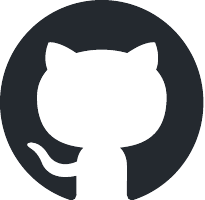}} & \textbf{Code} & \href{https://github.com/SakanaAI/DroPE}{\path{https://github.com/SakanaAI/DroPE}} \\
        \end{tabular}
    \end{center}
\end{abstract}

\section{Introduction}\label{sec:1introduction}

\begin{wrapfigure}[16]{r}{0.55\textwidth}
    \vspace{-15pt}
    \centering
    \includegraphics[width=0.55\textwidth]{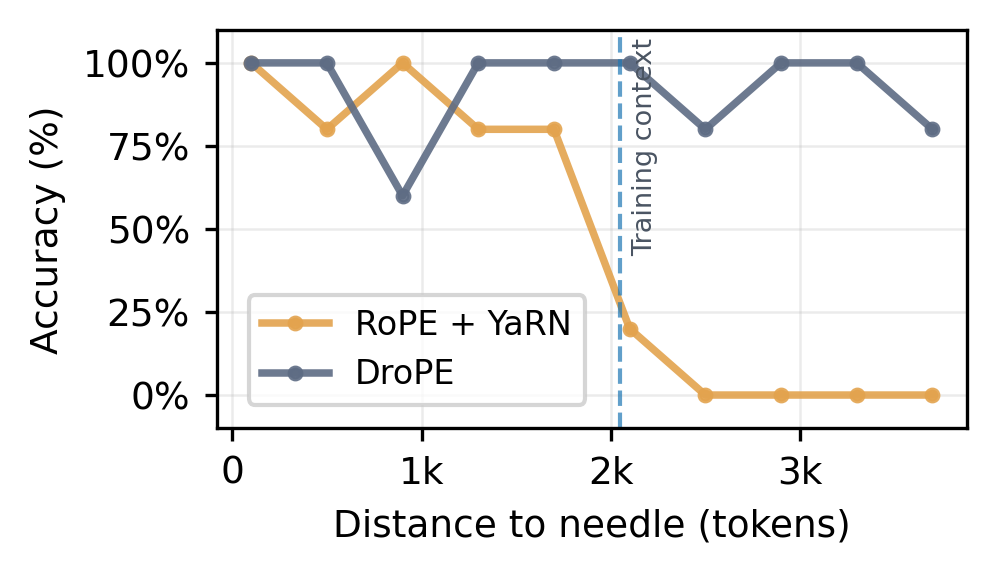}
    \vspace{-18pt}
    \caption{\textbf{DroPE generalizes zero-shot to long sequences.} Needle-in-a-haystack retrieval accuracy on sequences at $2\times$ the original context length with \emph{no long context training} (zero-shot context extension).}\label{fig:drope_v_yarn_niah}
\end{wrapfigure}

Transformers established themselves as the predominant architecture for training foundation models at unprecedented scale in language and beyond~\citep{gpt3, gemini, alphafold, vit}. The defining feature of transformers is abandoning explicit architectural biases such as convolutions and recurrences in favor of highly general self-attention layers~\citep{vaswani2017attention}, while injecting positional information about the sequence through positional embeddings (PEs) and causal masking. However, despite significant efforts to scale attention to long sequences on modern hardware~\citep{flashattention, ringattention, blockattention}, this powerful layer is inherently bottlenecked by quadratic token-to-token operations, which makes pretraining at long sequence lengths computationally intractable at scale. As a result, enabling models to use contexts beyond their pretraining length \emph{without additional long-context fine-tuning} (i.e., ``zero-shot context extension'') has emerged as a central challenge for the next generation of foundation models~\citep{alibi, chi2023attention}.

\looseness=-1
When inference sequence lengths exceed the pretraining context, the performance of modern transformer-based LMs degrades sharply. This is directly caused by their use of explicit PEs such as the ubiquitous rotary positional embeddings (RoPE)~\citep{rope}, which become out-of-distribution at unseen sequence lengths. To address this issue, careful scaling techniques that adapt RoPE frequencies on longer sequences were introduced~\citep{pi, ntk, yarn, longrope}. However, despite their popularity, these methods still rely on an expensive, long-context finetuning phase to \emph{meaningfully} use tokens beyond the original sequence length, failing to generalize out of the box~\citep{acontrolledstudy}. Beyond RoPE transformers, alternative architectures and positional embedding schemes have shown early promise in reducing costs by attenuating the underlying quadratic computational burden~\cite{performer, linformer, nystromformer, bigbird} or maintaining better out-of-context generalization~\citep{nope, rope2nope, swangpt}. Yet, these parallel efforts are still far from challenging established pipelines, introducing notable performance and stability trade-offs that prevent wide adoption.

\begin{figure}[t]
    \centering
    \includegraphics[width=0.85\textwidth]{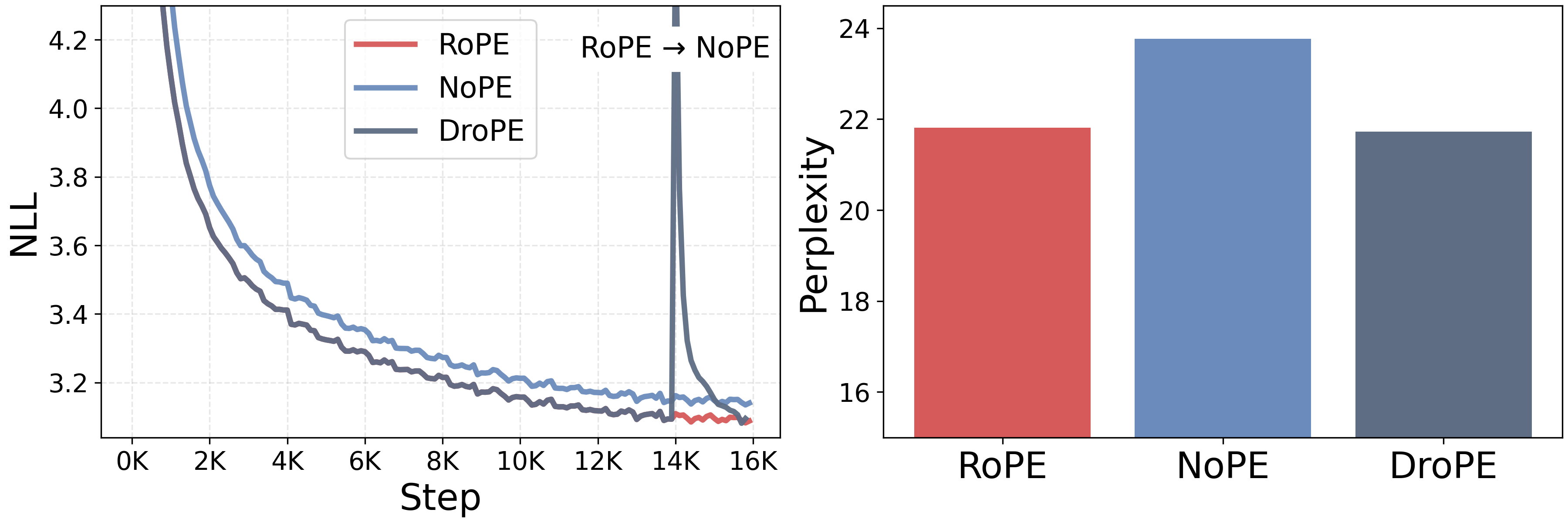}
    \vspace{-2pt}
    \caption{\textbf{\ouralgo matches RoPE's in-context perplexity.} We compare three training recipes: (1) a RoPE transformer trained for 16K steps (16B tokens), (2) a NoPE transformer trained for 16K steps, and (c) a \ouralgo transformer obtained by training the 14K-step RoPE checkpoint for 2K additional steps. The \ouralgo recipe matches the RoPE transformer's loss within 2K steps and achieves lower final perplexity than the NoPE-from-scratch baseline under the same budget.}\label{fig:loss_and_ppl}
\end{figure}

\looseness=-1
In this work, we challenge the conventional role of RoPE in language modeling, and propose to overcome this inherent trade-off by \emph{\textbf{Dro}pping the \textbf{P}ositional \textbf{E}mbeddings} (\ouralgo) of LMs after pretraining. Our method is based on three key theoretical and empirical observations. First, explicit positional embeddings significantly facilitate pretraining convergence by baking in an important inductive bias that is difficult to recover from data alone. Second, over-reliance on positional embeddings is precisely what prevents test-time generalization to sequences of unseen length, with RoPE-scaling context extension methods focusing on recent tokens instead of ones deeper in the context to retain perplexity. Third, explicit PE is not an inherent requirement for effective language modeling and can be \emph{removed after pretraining}, following a short recalibration phase which is performed at the \emph{original context length}. 

Empirically, \ouralgo models generalize zero-shot to sequences far beyond their training context, marking a sharp contrast to traditional positional scaling techniques. Moreover, we show that adapting RoPE models with \ouralgo does not compromise their original in-context capabilities, preserving both perplexity and downstream task performance. Our findings hold across LMs of different architectures and sizes up to 7B parameters pretrained on trillions of tokens, establishing a new standard for developing robust and scalable long-context transformers.

\textbf{Contributions.} In summary, our main contributions are as follows:
\begin{enumerate}[label=(\arabic*)]
    \item In Section~\ref{sec:pe_is_important}, we provide empirical and theoretical analysis of the role of positional embeddings in LM \emph{training}, showing their importance in significantly accelerating convergence.
    \item In Section~\ref{sec:rope_scaling_fails}, we discuss why RoPE-scaling methods fail to reliably attend across far-away tokens when evaluated zero-shot on long sequences, showing that these approaches inevitably shift attention weights, hindering the model's test-time behavior.
    \item In Section~\ref{sec:4method}, we introduce \ouralgo, a new method that challenges the conventional role of positional embeddings in transformers, motivated by our empirical and theoretical analyses of its role as a transient but critical training inductive bias.
    \item We demonstrate that \ouralgo enables \emph{zero-shot generalization} of pretrained RoPE transformers far beyond their original sequence length, \emph{without any long-context finetuning}. \ouralgo can be incorporated \emph{at no extra cost} into established training pipelines, and can be used to inexpensively empower \textit{arbitrary pretrained LLMs in the wild}.
\end{enumerate}

We \textbf{share our \href{https://github.com/SakanaAI/DroPE}{code}} to facilitate future work and extensions toward developing foundation models capable of handling orders-of-magnitude longer contexts.
\section{Preliminaries}\label{sec:2prel}

\textbf{Self-attention.} Let $h_1,\dots,h_T\in\sR^{d}$ be the representations fed into a multi-head attention block. Queries $q_i$, keys $k_i$, and values $v_i$ are computed by projecting the inputs $h_i$  via linear layers $W_Q$, $W_K$, and $W_V$. The \emph{attention} operation then computes a $T \times T$ matrix of attention scores $s_{ij}$ and then weights $\alpha_{ij}$ between all pairs of sequence positions, and reweighs value vectors:
\begin{equation}\label{eq:attention}
    s_{ij}=\tfrac{1}{\sqrt{d_k}}q_i^\top k_j, \quad
    \alpha_{ij}=\softmax(s_{i1},\dots,s_{ii})_j,\quad
    z_i=\sum_{j\leq i}\alpha_{ij}v_j,
\end{equation}
where $d_k$ is the head dimension. A multi-head attention block computes multiple attention outputs $z_i^{(1)}, \dots, z_i^{(H)}$, concatenates them, and projects to the model dimension: $o_i = W_O[z_i^{(1)}, \dots, z_i^{(H)}]$.

\textbf{Language and positional embeddings.} State-of-the-art autoregressive transformer LMs use information about sequence positions provided both implicitly via causal masking of the attention scores\footnote{Note the softmax in Equation~\ref{eq:attention} is taken on the first $i$ tokens, implementing a causal mask.}, and explicitly with positional embeddings. In particular, the modern literature has settled on the Rotary PE (RoPE) scheme~\citep{rope}, providing relative positional information to each attention head by rotating $q_i$ and $k_j$ in 2D chunks before the inner product in Equation~\ref{eq:attention}:
\begin{equation}\label{eq;rope_attnetion}
    s_{ij} = \tfrac{1}{\sqrt{d_k}}(R^iq_i)^\top (R^j k_j)=  \tfrac{1}{\sqrt{d_k}}q_i^\top R^{j - i} k_j, \qquad R=\operatorname{block\text{-}diag} \big(R(\omega_1),\ldots,R(\omega_{d_k/2})\big).
\end{equation}
Here, each $R(\omega_m)\in \R^{2\times2}$ is a planar rotation of angle $\omega_m=b^{-2(m-1)/d_k}$ acting on the $(2m,2m+1)$ subspace of $q_i$ and $k_j$. The base $b$ is commonly taken to be $10{,}000$.

\textbf{Context extension for RoPE.} Given the rapidly growing costs of self-attention, adapting LMs for longer sequences than those seen during training has been a longstanding open problem. To this end, prior context-extension methods introduce targeted rescaling of the RoPE frequencies in Equation~\ref{eq;rope_attnetion} to avoid incurring unseen rotations for new sequence positions. Formally, let the training and inference context lengths be $C_\mathrm{train} < C_\mathrm{test}$, and define the extension factor $s= C_\mathrm{test} / C_\mathrm{train}$. Context extension methods such as PI~\citep{pi}, RoPE-NTK~\citep{ntk}, and the popular YaRN~\citep{yarn} define new RoPE frequencies $\omega_m' = \gamma_m \omega_m$ with scaling factors:
\begin{equation}\label{eq:rope_scaling}
    \gamma_m^\mathrm{PI} = \tfrac{1}{s}, \quad
     \gamma_m^\mathrm{NTK} = \left(\tfrac{1}{s}\right)^{\frac{2m}{d_k -2}}, \quad\text{and}\quad \gamma_m^\mathrm{YaRN} = (1-\kappa_m)\tfrac{1}{s} + \kappa_m,
\end{equation}

where $\kappa_m \in [0, 1]$ interpolates between 0 and 1 as the base frequency $\omega_m$ grows (see Appendix~\ref{apx:extended_preliminaries}). These methods, referred to as \emph{RoPE-scaling}, still require additional finetuning on long sequences, and don't generalize to long-context downstream tasks out of the box~\citep{controlledstudylongcontext_rush}. 

\looseness=-1
\textbf{NoPE transformers.} In a parallel line of work, there have been efforts to train transformers without PEs, commonly referred to as NoPE architectures \citep{haviv2022transformer, nope}, to avoid the need for rescaling RoPE frequencies. While NoPE was shown to be a viable LM architecture, it has failed to gain traction due to degraded performance~\citep{haviv2022transformer, rope2nope} compared to RoPE architectures. For an in-depth introduction to the above concepts, see Appendix~\ref{apx:extended_preliminaries}.
\section{Explicit positional embeddings are beneficial for training}\label{sec:pe_is_important}

\begin{wrapfigure}[15]{r}{0.44\textwidth}
    \vspace{-10pt}
    \centering
    \includegraphics[width=0.44\textwidth]{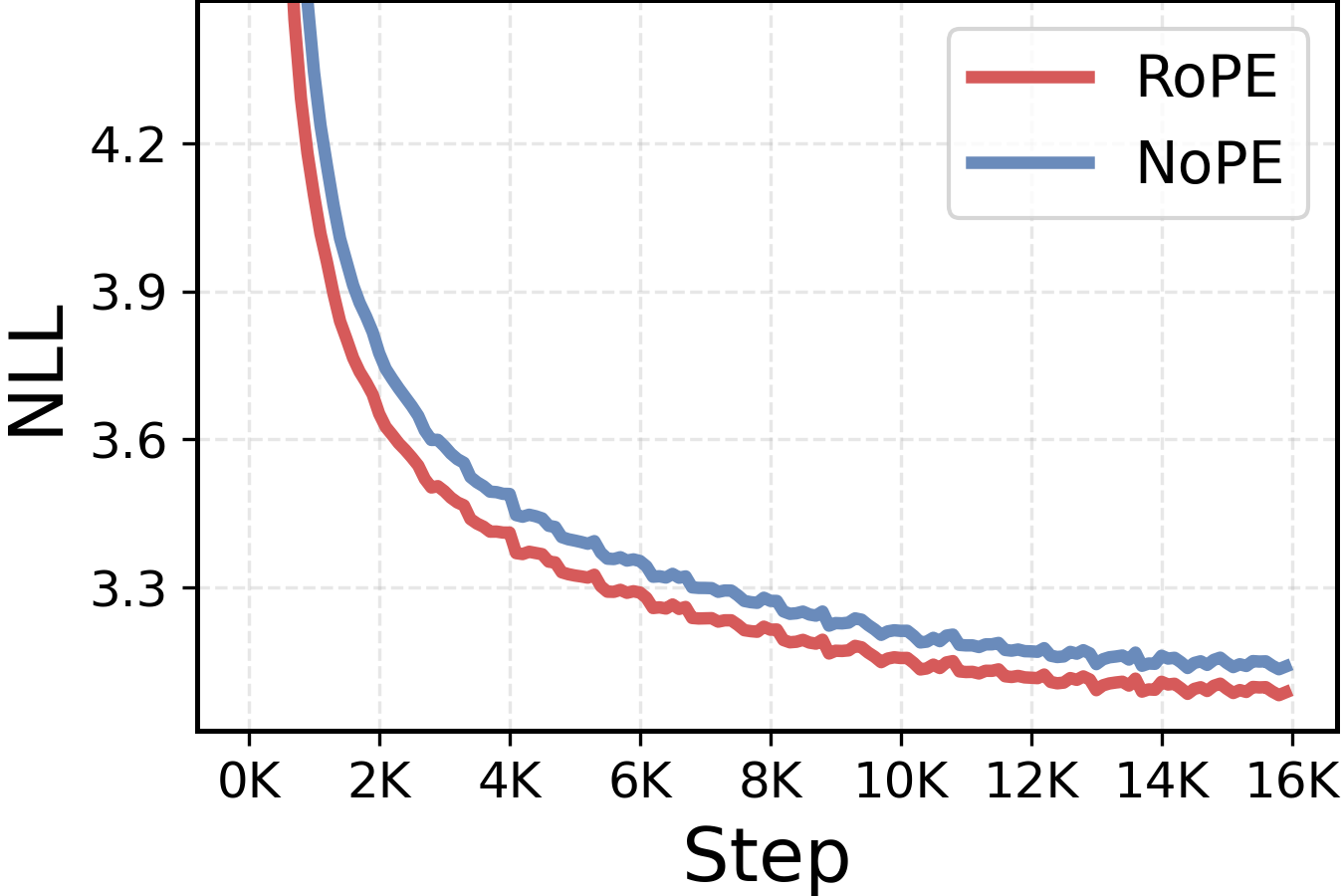}
    \vspace{-14pt}
    \caption{\textbf{RoPE outperforms NoPE.} Training loss curves for a RoPE and NoPE transformers on 16B fineweb tokens. RoPE outperforms NoPE \emph{throughout training}.}\label{fig:rope_nope_loss_comp}
\end{wrapfigure}
\looseness=-1
While NoPE transformers were shown to be \emph{expressive} enough for effective sequence modeling~\citep{haviv2022transformer, nope}, we find that they \emph{consistently underperform RoPE architectures} throughout our experiments. As illustrated in Figure~\ref{fig:rope_nope_loss_comp}, NoPE transformers maintain visibly worse perplexity \emph{throughout training}. These empirical results are consistent with past literature~\citep{haviv2022transformer, rope2nope}, yet the reasons why positional embeddings are key for effective language model \emph{training} have never been fully understood.


From a purely mechanistic perspective, even without explicit positional embeddings, NoPE transformers can exploit the causal mask to \emph{encode} positional information, maintaining the same expressivity as their RoPE counterparts~\citep{haviv2022transformer, nope}. Specifically, \citet{nope} prove that the first attention layer in a NoPE transformer can \emph{perfectly reconstruct} sequence positions, and subsequent layers can emulate the effects of relative or absolute positional embeddings. As detailed in Section \ref{sec3:subsec:theortical_der}, rather than looking at theoretical expressivity, we investigate this empirical performance discrepancy from an \emph{optimization} perspective, providing theoretical analysis of the positional bias of NoPE transformers during training. The theoretical and empirical analysis in this section can be summarized in the following observation.

\begin{mdframed}[style=observation]
    \begin{observation}\label{find:pe_is_crucial}
        
        Positional information and attention non-uniformity, which are crucial for sequence modeling, develop at a \textbf{bounded rate} in NoPE transformers. In contrast, explicit PE methods, such as RoPE, provide a strong bias from the outset and facilitate the propagation of positional information, resulting in faster training.
    \end{observation}
\end{mdframed}
At a high level, our analysis focuses on the rate at which NoPE and RoPE transformers can develop \emph{positional bias} in their self-attention heads, which captures their non-uniformity. We quantify attention positional bias as a linear functional on the attention map:
\begin{mdframed}[style=theorem]
    \begin{definition}[Attention positional bias]
    \label{sec3:defn:positional_bias}
        Given centered positional weights $c_{ij} \in \sR$ with $\sum_{j \leq i} c_{ij} = 0$, the \emph{positional bias} of the attention weights $\alpha_{ij}$ is
        \vspace{-5pt}
        \begin{equation*}
            \abias^c(\alpha)= \frac{1}{T}\sum_{i=1}^T\sum_{j \leq i} c_{ij} \alpha_{ij}.
        \end{equation*}
    \end{definition}
\end{mdframed}
\looseness=-1
Attention heads with a strong positional bias would maximize the average value of $\abias^c$ across input sequences for some weights $c$. For example, a ``diagonal'' attention head, focusing mass on the current token, is exactly the maximizer of $\abias^c$, with $c_{ij}$ having $1$s on the diagonal and $-\tfrac{1}{i-1}$ otherwise.

To validate the theory behind Observation~\ref{find:pe_is_crucial}, we empirically compare the gradients of the attention positional bias functional in attention heads of RoPE and NoPE transformers. Specifically, we measure the average gradient norm at initialization in the direction of two common language modeling patterns: diagonal attention heads, placing mass on the current token, and off-diagonal heads, capturing immediate previous token context. As illustrated in Figure~\ref{fig:abias_grads}, the gradient magnitudes of NoPE transformers are far lower than those of RoPE transformers, with the gap between the two growing in deeper layers. This means that diagonal and off-diagonal heads are slower to develop under NoPE, reflecting its difficulty in recovering positional information. In the next section, we theoretically analyze the causes of this gradient norm gap.

\begin{figure}[t]
    \vspace{-10pt}
    \centering
    \begin{subfigure}[b]{0.49\textwidth}
        \centering
        \includegraphics[width=0.95\textwidth]{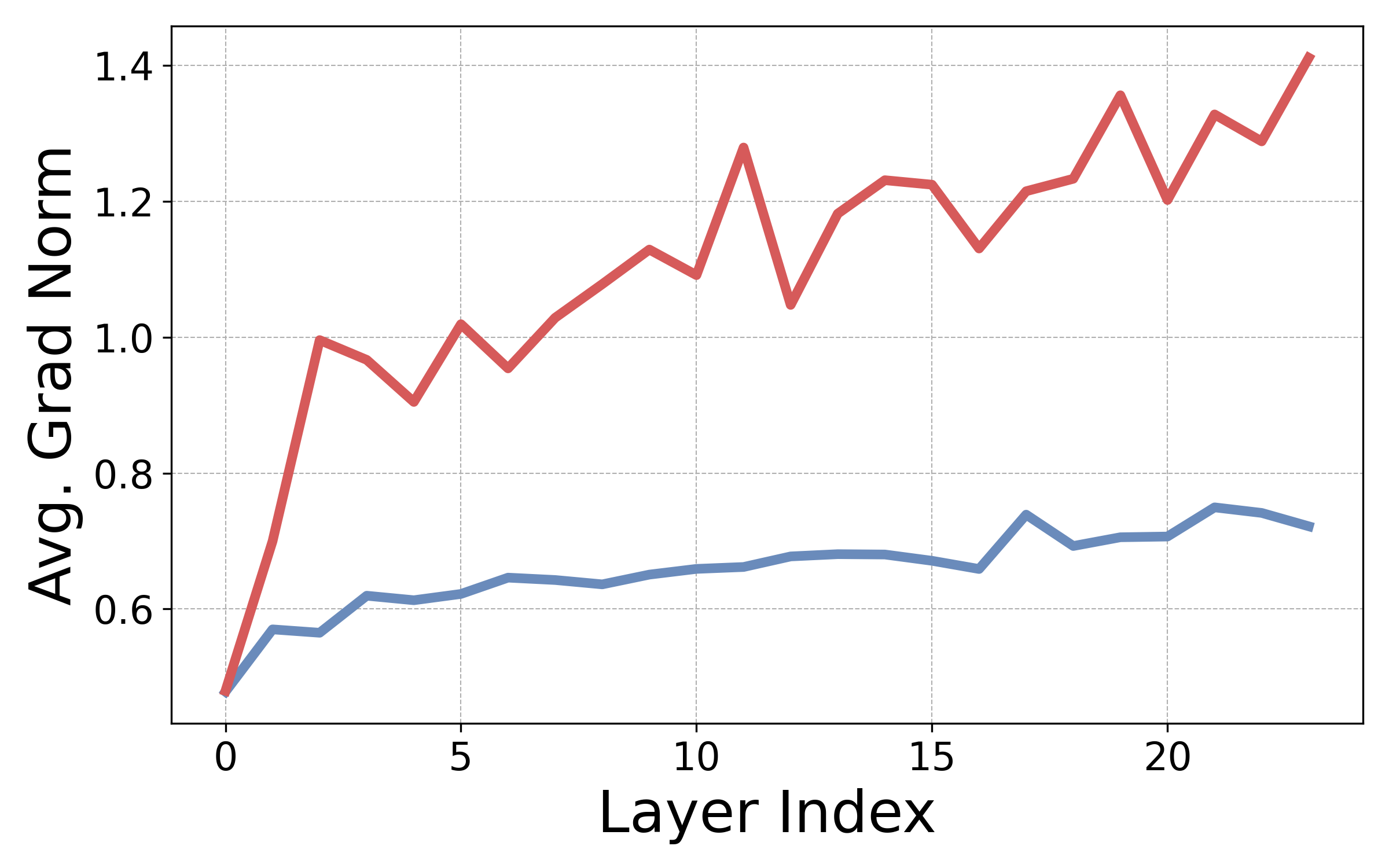}
        \caption{Diagonal head bias.}
        \label{fig:diagonal}
    \end{subfigure}
    \hfill
    \begin{subfigure}[b]{0.49\textwidth}
        \centering
        \includegraphics[width=0.95\textwidth]{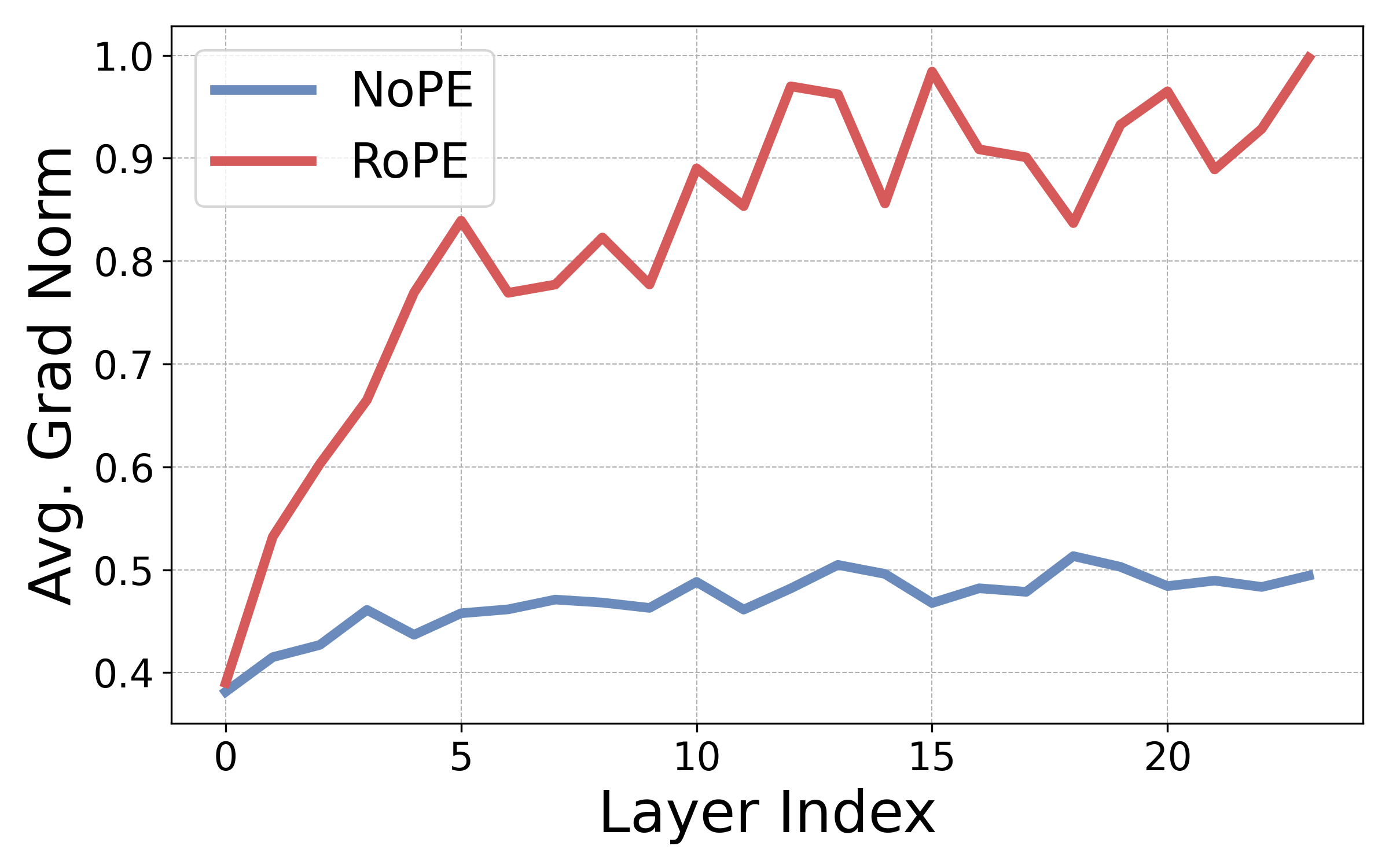}
        \caption{Off-diagonal head bias.}
        \label{fig:off_diagonal}
    \end{subfigure}
    \caption{\textbf{RoPE transformers have higher positional bias gradients at initialization.} We compare the average norm of $\abias^c$ across layers, for RoPE and NoPE transformers. In~\ref{fig:diagonal} we plot the gradient norms of positional bias towards a diagonal head, and in~\ref{fig:off_diagonal}, we take bias towards previous token attention, off-diagonal head. In both cases, the gradient norm is consistently higher for RoPE across layers, meaning that RoPE heads can learn these patterns faster.}\label{fig:abias_grads}
\end{figure}

\subsection{Theoretical analysis}\label{sec3:subsec:theortical_der}
We detail our findings, summarized in Observation~\ref{find:pe_is_crucial}, with a series of formal results, bounding the rate at which positional bias can develop early in training. We provide full proofs and an extended analysis of these results in Appendix~\ref{appendix:extended_theory}. Throughout this section, we study the sensitivity of the attention positional bias $\abias^c$ to the transformer's parameters and interpret $\norm{\nabla_\params \abias^c}$ as bounding the rate at which non-uniform attention patterns can emerge during training.

\textbf{Warm-up: NoPE transformers break on constant sequences.} Before moving to the main theoretical result, we consider a motivating example that illustrates NoPE transformers' training difficulties. Because attention forms a convex combination of value vectors, an attention head applied to a sequence of identical tokens $x_1 = \cdots = x_T$ produces identical outputs at every position. Moreover, since normalization layers, MLP blocks, and residual connections act \emph{pointwise} on tokens, this uniformity propagates through the network. In a NoPE transformer, this means the attention logits are constant over all $j \leq i$, hence the post-softmax attention probabilities are uniform. Consequently, the model cannot induce any positional preference and $\abias^c \equiv 0$ for \emph{any} positional weights $c$. 
\begin{mdframed}[style=theorem]
    \begin{restatable}{proposition}{ConstSeq}\label{prop:const_seq}
        
        Let $\mathsf{M}$ be a NoPE transformer. If the input sequence $x = (x_1, \dots, x_T)$ is comprised of identical tokens $x_1 = \cdots = x_T$, then (1) \textbf{all} attention heads are uniform: $\alpha_{ij} = \frac{1}{i}$, (2) query and key gradients vanish: $\partial \gL / \partial W_Q = \partial \gL / \partial W_K = 0$, (3) for all heads and any positional weights $\abias^c = 0$, $\nabla_\params \abias^c = \mathbf{0}$ , and (4) the output is constant: $\mathsf{M}(x)_1 = \cdots = \mathsf{M}(x)_T$.
    \end{restatable}
\end{mdframed}

The explicit positional information injected into attention heads in RoPE transformers circumvents this issue. Enabling non-zero $\abias^c$ gradients even on constant sequences.
\begin{mdframed}[style=theorem]
    \begin{restatable}{proposition}{RoPELowerbound}\label{prop:rope_lowerbound}
        For a non-trivial RoPE attention head, even if the input sequence is constant, there are positional weights $c$, for which $\abias^c > 0$, and $\norm{\nabla_\params \abias^c} > 0$.
    \end{restatable}
\end{mdframed}

\textbf{NoPE transformers propagate embedding uniformity.} At initialization, the entries of the embedding matrix are drawn i.i.d. from a distribution with a fixed small variance (commonly, $\sigma^2 = 0.02$). Therefore, the token embeddings are close to uniform at the beginning of training. The next theorem shows that for NoPE transformers, this uniformity persists throughout the network, and bounds the attention positional bias $\abias^c$ and its gradients. 
\begin{mdframed}[style=theorem]
    \begin{restatable}{theorem}{Propagation}\label{thm:propagation}
        Define the he prefix-spread of the hidden states at layer $l$ as 
        \begin{equation*}
            \Delta_h^{(l)} := \max_{1 \leq j \leq i \leq T} \bnorm{\bar{h}^{(l)}_i - h^{(l)}_j}, \quad \text{where} \quad \bar{h}_i^{(l)} := \frac{1}{i}\sum_{j \leq i} h_j^{(l)}.
        \end{equation*}
        For NoPE transformers, there exists $\varepsilon > 0$ and constants $C_1$, $C_2$, and $C_3$ such that if the initial embeddings $\Delta_h^{(1)} \leq \varepsilon$, then for all layers $l \leq L$:
        \begin{equation*}
            \Delta_h^{(l)} \leq C_1 \varepsilon, \qquad \big|\abias^c\big| \leq C_2 \varepsilon, \qquad \Bnorm{\partial \abias^c / \partial W_Q}, \Bnorm{\partial \abias^c/\partial W_K} \leq C_3 \varepsilon,
        \end{equation*}
        with high probability over the initialization distribution. The constants only depend on the number of layers and heads, and \textbf{not} on the sequence length.
    \end{restatable}
\end{mdframed}
The main idea in the proof of Theorem~\ref{thm:propagation} is that uniformity in the embeddings causes uniformity in the attention maps, so $\alpha_{ij} \approx 1 /i$. Uniform mixing of tokens cannot increase the prefix spread; thus, uniformity persists throughout the network. This result explains the discrepancy between RoPE and NoPE transformers illustrated in Figure~\ref{fig:abias_grads}. 

In summary, we demonstrate that while NoPE attention can learn positional bias, attention non-uniformity develops slowly early in training due to bounded $\abias^c$ gradients at initialization. 



\section{RoPE prevents effective zero-shot context extension}\label{sec:rope_scaling_fails}
\begin{figure}[t]
    \centering
    \includegraphics[width=0.85\linewidth]{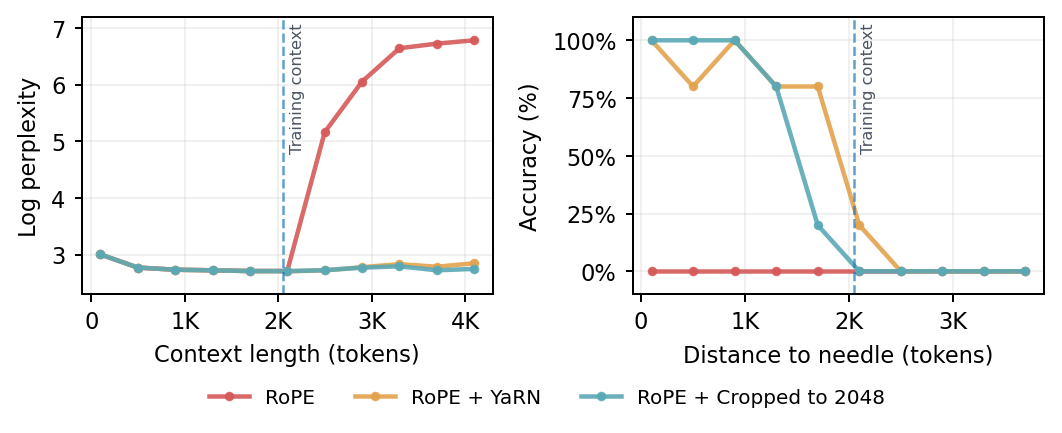}
    \caption{\textbf{YaRN crops effective retrieval context.} We compare RoPE's and YaRN's perplexity and NIAH performance at up-to $2\times$ the original context length against a baseline that \emph{crops} the input sequence to the training context length. Both YaRN and the cropped baseline can maintain perplexity on sequences exceeding the training context length, but are unable to retrieve information placed far away from the query.} \label{fig:rope_yarn_crop}
\end{figure}
\begin{wrapfigure}[15]{r}{0.45\textwidth}
    \centering
    \vspace{-12pt}
    \includegraphics[width=0.45\textwidth]{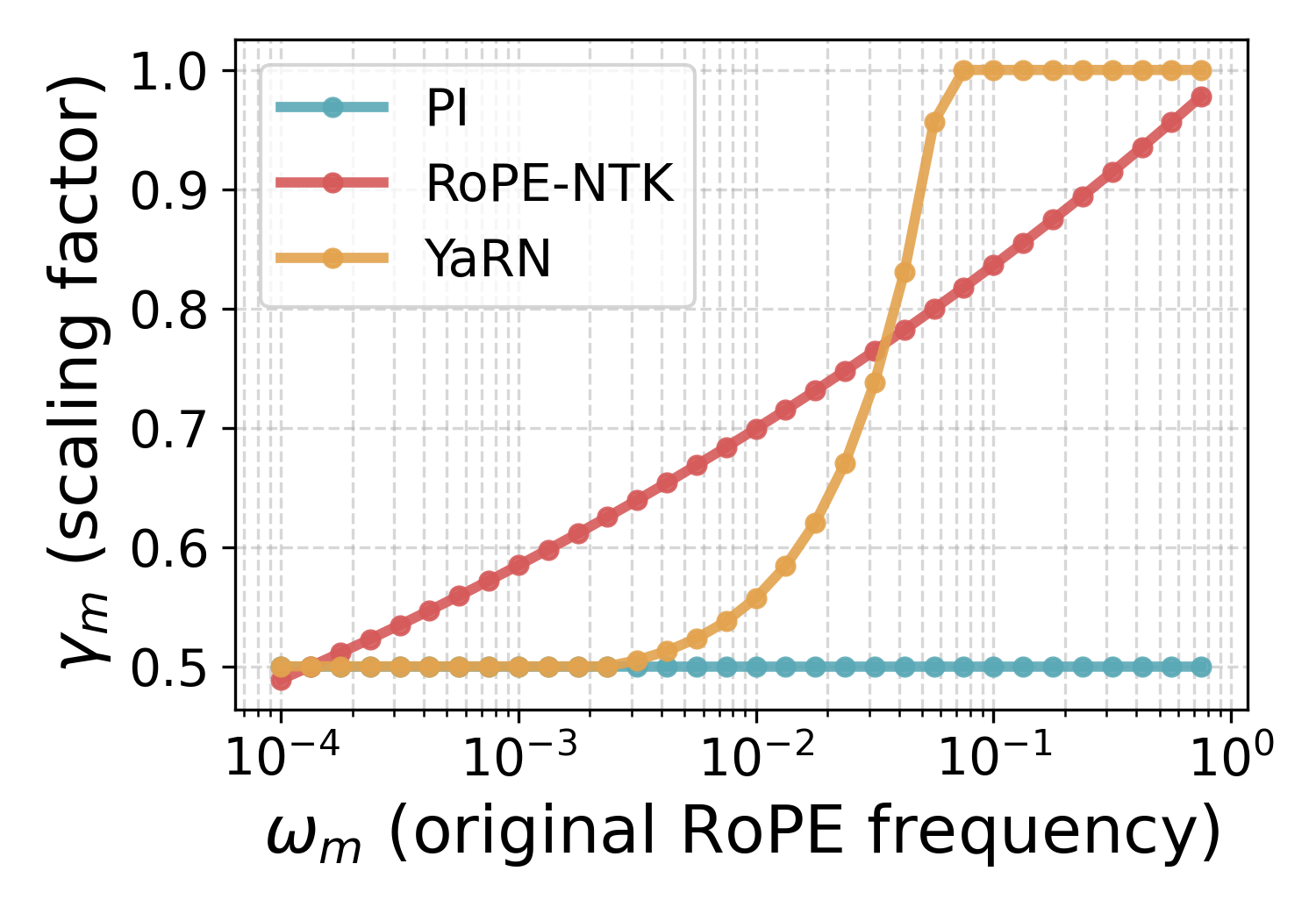}
    \vspace{-18pt}
    \caption{RoPE frequency scaling under PI, NTK-aware scaling (RoPE-NTK), and YaRN, with scaling factor $s = 2$.}\label{fig:freq_scaling}
\end{wrapfigure}

State-of-the-art RoPE scaling methods fail to effectively generalize to sequences longer than those seen in training without additional long-context finetuning. While YaRN and other popular frequency scaling techniques do avoid perplexity degradation on long-context sequence~\citep{ntk, yarn}, they exhibit sharp performance drops on downstream tasks whenever important information is present deep in the sequence, beyond the training context~\citep{controlledstudylongcontext_rush, lost_in_the_middle}. We empirically demonstrate this phenomenon, comparing the perplexity and needle-in-a-haystack (NIAH)~\citep{niah_repo, ruler} performance of a RoPE transformer scaled with YaRN and to a cropped context baseline. As illustrated in Figure~\ref{fig:rope_yarn_crop}, YaRN's zero-shot behavior closely matches that of simply \textit{cropping} the sequence length to the pretraining context, maintaining constant perplexity but ignoring information present outside the cropped window.

The cause of this limitation lies in the way context extension methods scale different RoPE frequencies. As detailed in Section~\ref{sec:2prel}, elaborated on in Appendix~\ref{apx:extended_preliminaries}, and illustrated in Figure~\ref{fig:freq_scaling}, the scaling factors of PI~\citep{pi}, RoPE-NTK~\citep{ntk}, and YaRN~\citep{yarn} have a strong effect on \emph{low frequencies}. In Section~\ref{sec4:subsec:why_inevitable}, we discuss why this aggressive scaling of low frequencies leads to the observed failures, yielding our second observation.
\begin{mdframed}[style=observation]
    \begin{observation}\label{find:yarn_is_bad_at_recall}
        RoPE-scaling methods \textbf{must} compress low frequencies to keep positional phases in-distribution. This, in turn, shifts semantic attention heads at large relative distances, causing the observed failures on downstream tasks, preventing zero-shot context extension. 
    \end{observation}
\end{mdframed}

\subsection{Why extrapolation failure is inevitable}
\label{sec4:subsec:why_inevitable}
\begin{wrapfigure}[10]{r}{0.5\textwidth}
    \centering
    \vspace{-30pt}
    \includegraphics[width=0.5\textwidth]{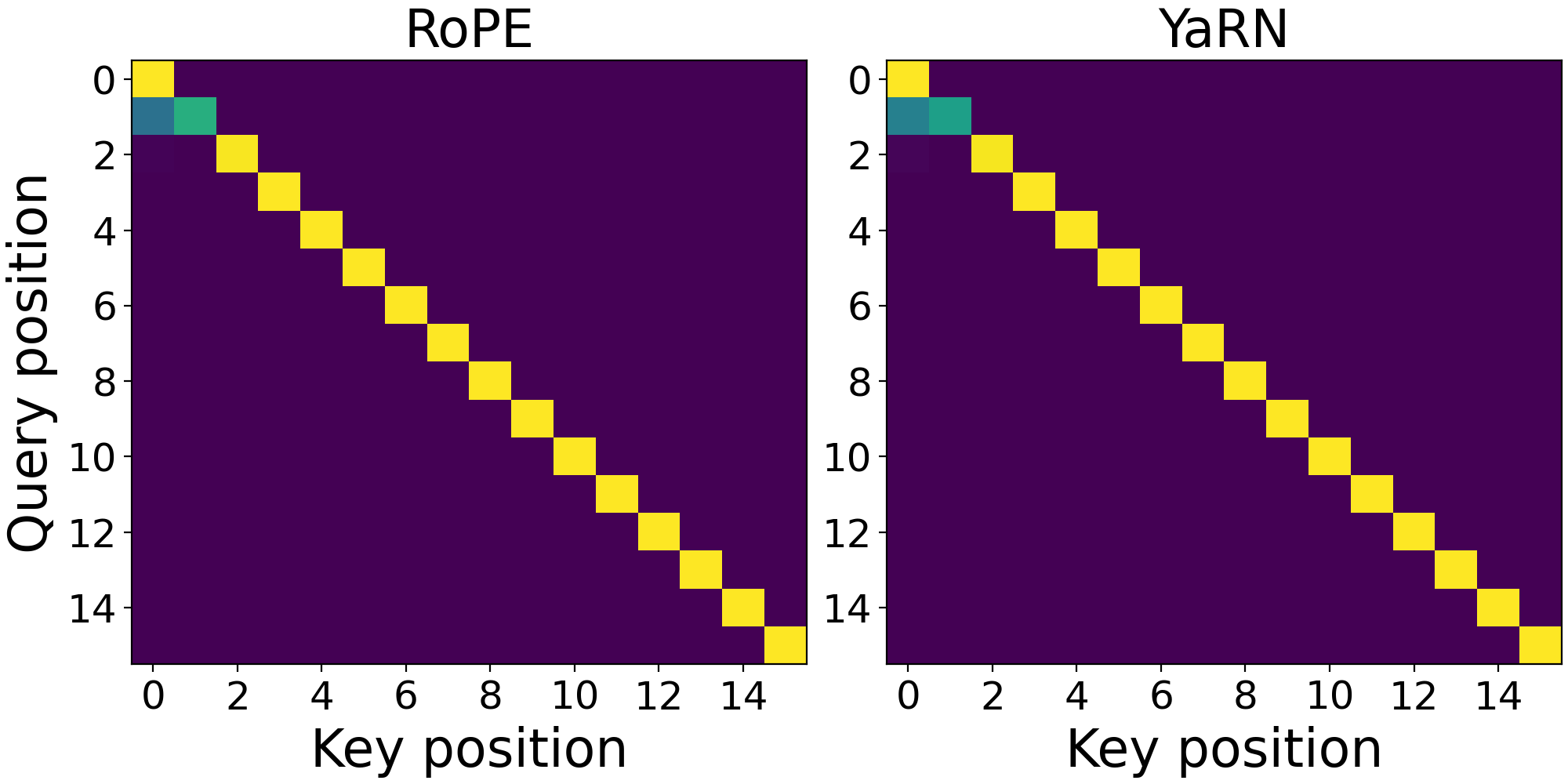}
    \vspace{-20pt}
    \caption{\textbf{RoPE scaling preserves average attention in positional heads.}}\label{fig:yarn_pos_heads}
\end{wrapfigure}

\textbf{Effect of RoPE scaling.} RoPE scaling methods modify the frequencies at inference time to evaluate sequences that are longer than those seen during pretraining. In each $(2m,2m{+}1)$ subspace, the RoPE phase at relative distance $\Delta$ is $\phi_m(\Delta)=\omega_m \Delta$, so scaling the frequency to $\omega'_m=\gamma_m \omega_m$ is equivalent to using a phase $\phi_m'(\Delta) = \gamma_m \omega_m \Delta$. As illustrated in Figure~\ref{fig:freq_scaling}, most scaling methods leave high frequencies nearly unchanged ($\gamma_m\approx 1$) but \emph{all of them} compress the low frequencies ($\gamma_m \approx 1/s$). As demonstrated both theoretically and empirically in~\citet{barbero2024round}, high RoPE frequencies are primarily used by \emph{positional heads}, with attention patterns based on relative token positions (e.g., diagonal or previous-token heads). In contrast, low frequencies are predominantly used by \emph{semantic heads} that attend based on query/key content. Consequently, positional heads are largely unaffected by scaling, but semantic attention is shifted. Moreover, the effect on low-frequency dominated semantic heads is exacerbated for distant tokens, since the relative phase $\phi_m(\Delta)$ is larger, and thus the $1/s$ scaling factor has a greater effect. In other words, scaling \emph{warps} low-frequency phases, shifting long-range attention in precisely the subspaces most used for semantic matching.


In Figure~\ref{fig:yarn_pos_heads} and Figure~\ref{fig:yarn_mid_sequence}, we illustrate this behavior in practice. We start by selecting a positional attention head in a pretrained \textsc{Qwen2.5-0.5B} model by examining its average attention positional bias (Definition~\ref{sec3:defn:positional_bias}) across layers. In Figure~\ref{fig:yarn_pos_heads}, we show the average attention weights in this positional head under YaRN scaling with $s=2$. Because high frequencies, which are least affected by YaRN, dominate positional heads, the average attention profiles are similar. In Figure~\ref{fig:yarn_mid_sequence}, we then contrast this behavior with that of a semantic head for a long needle-in-a-haystack sequence, plotting the average attention of the last token (query) with tokens around the needle (keys). YaRN's aggressive scaling of low frequencies substantially shifts attention mass across tokens, reflecting the impact of frequency compression at longer ranges.

\begin{figure}[t]
    \centering
    \includegraphics[width=\textwidth]{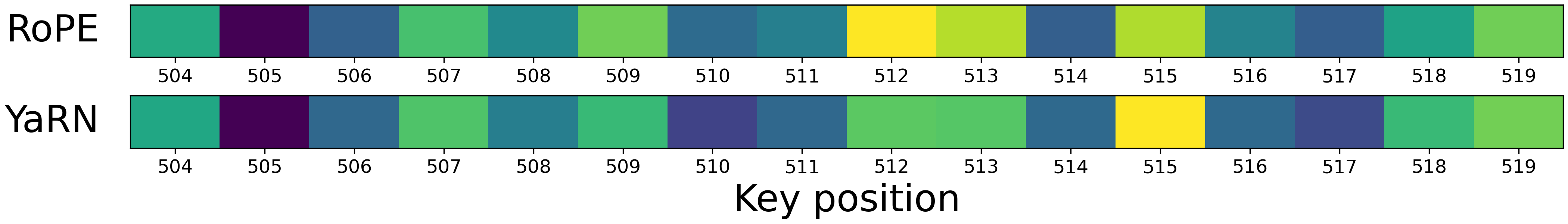}
    \caption{\textbf{RoPE scaling shifts semantic attention mass.} Attention weights of the last token (query) with tokens from a retrieval target (keys) in a semantic head evaluated on a NIAH probe. Since the head uses low frequencies and the relative distance is non-trivial, the impact of YaRN is substantial, shifting attention mass between tokens.}\label{fig:yarn_mid_sequence}
\end{figure}

\textbf{Why this is inevitable.}
In a standard RoPE setup, low-frequency phases never make a full cycle over the original context length: $\phi_m(C_{\mathrm{train}})=\omega_m C_{\mathrm{train}} < 2 \pi$ for small $\omega_m$. E.g. for a standard RoPE base $b = 10^4$, a transformer with head dimension $d_k = 64$, will have at least five low frequencies for which $\phi_m(C_\mathrm{train}) < 2 \pi$, even at a training context of $C_\mathrm{train} = 32{,}000$. If we leave $\omega_m$ unchanged at an extended length $C_{\mathrm{test}} > C_{\mathrm{train}}$, the new maximal relative phase $\phi_m(C_{\mathrm{test}})$ is pushed outside the training regime and becomes out of distribution for the head. Therefore, to constrain phases to remain in range, any scaling method must choose $\gamma_m \le \tfrac{C_{\mathrm{train}}}{C_{\mathrm{test}}} = \tfrac{1}{s}$, which becomes increasingly small as the extension factor $s$ grows. In other words, when applying a RoPE transformer to sequences longer than those seen in training, any post-hoc scaling method \emph{must} compress the low frequencies. But this compression, in turn, shifts attention weights at long relative distances.

\section{\ouralgo: Dropping positional embeddings after pretraining}\label{sec:4method}

\looseness=-1
Taken together, Observations~\ref{find:pe_is_crucial} and~\ref{find:yarn_is_bad_at_recall} imply that providing explicit positional information with PE is a key component for effective LM training, but is also a fundamental barrier to long-context generalization. This raises a natural question: is it possible to harness the inductive bias from positional embeddings \textit{exclusively} during training? We answer in the affirmative. In this section, we demonstrate that it is possible to drop all positional embeddings from a pretrained transformer and quickly recover the model's in-context capabilities with a brief recalibration phase. Most notably, this simple new procedure (termed \ouralgo) unlocks strong \textit{zero-shot} long context generalization to unseen sequence lengths, far beyond highly-tuned RoPE extensions and prior alternative architectures.
\begin{mdframed}[style=observation]
    \begin{observation}\label{find:drope}
        Positional embeddings can be \textbf{removed after pretraining}, allowing LMs to generalize \textbf{zero-shot} to \textbf{unseen sequence lengths} without compromising their in-context performance after short recalibration on a fraction of the training tokens at the original context size.
    \end{observation}
\end{mdframed}

\subsection{Large-scale empirical evaluation}
We extensively validate \ouralgo across different LM and dataset scales, showing it outperforms prior approaches both as a \emph{zero cost} integration into pretraining recipes and as an inexpensive way to adapt \emph{any LM in the wild} already pretrained on trillions of tokens. For all experiments in this paper, we provide full implementation details of each evaluated architecture and optimization phase, including comprehensive hyperparameter lists in Appendix~\ref{appendix: experimental details}.

\looseness=-1
\textbf{Integrating \ouralgo at no extra cost.} For our first set of experiments, we train from scratch different LMs with half a billion parameters on 16B fineweb tokens~\citep{fineweb}, over twice the chinchilla-optimal rate~\citep{chinchilla}. We repeat this recipe for RoPE and NoPE transformers, as well as an ALiBi model~\citep{alibi} and an RNoPE-SWA model~\citet{rope2nope}, which are alternative architectures specifically aimed at long-context capabilities. We implement DroPE by taking the 14B tokens RoPE transformer checkpoint, removing positional embeddings from every layer, and resuming training for the final 2B tokens. Despite only recalibrating at the very end of training, at no extra cost, \ouralgo matches the final in-context validation perplexity of RoPE trained on the full 16B tokens, showing a clear edge over the NoPE baseline trained without positional embedding all the way (Figure~\ref{fig:loss_and_ppl}). We provide further analysis and ablations on the recalibration starting point in Appendix~\ref{apx:rec_ablation}, validating the importance of harnessing the inductive bias of RoPE for a substantial amount of training, in line with the core motivation of our new method.

\begin{wraptable}[14]{r}{0.55\linewidth}
    \centering
    \small
    \vspace{-8pt}
    \caption{\textbf{Zero-shot NIAH at $2\times$ training context.} Results are reported as a success rate over 500 trials.} \label{tab:niah_results}
    \resizebox{0.55\textwidth}{!}{
        \begin{tabular}{lccc}
            \toprule
            \textbf{Method} & \textbf{\shortstack{Multi- \\ Query}} & \textbf{\shortstack{Multi- \\ Key}} & \textbf{\shortstack{Multi- \\ Value}} \\
            \midrule
            RoPE transformer &  $0.0$ & $0.0$ & $0.0$ \\
            RoPE transformer + PI & $0.0$ & $0.0$ & $0.0$ \\
            RoPE transformer + RoPE-NTK & $21.1$ & $19.4$ & $16.5$ \\
            RoPE transformer + YaRN & $17.8$ & $0.5$ & $14.6$ \\
            \midrule
            ALiBi transformer & 5.2 & 0.0 & 1.1 \\
            NoPE transformer &  $9.2$ & $36.2$ & $21.4$ \\
            RNoPE-SWA transformer & $5.2$ & $25.6$ & $20.6$ \\
            \midrule
            \ouralgo transformer & $\mathbf{28.0}$ & $\mathbf{41.6}$ & $\mathbf{23.3}$ \\
            \bottomrule
        \end{tabular}
    }
\end{wraptable}

To evaluate the long-context generalization of each method, we select three tasks from the RULER benchmark~\citep{ruler}: (1) \emph{multi-query:} retrieve needles for several listed keys, (2) \emph{multi-key:} retrieve the needle for one specified key, and (3) \emph{multi-value:} retrieve all needles for one key with a single query. For the base RoPE transformer, we consider three context extension strategies: PI~\citep{pi}, NTK-RoPE~\citep{ntk}, and the popular YaRN~\citep{yarn} described in Section~\ref{sec:2prel} and Appendix~\ref{apx:extended_preliminaries}. In Table~\ref{tab:niah_results}, we report the success rate on each task at $2\times$ the training context length. Our \ouralgo model \emph{substantially outperforms all baselines} in each setting. While RoPE-NTK and YaRN also yield improvements to the original RoPE transformer, they consistently trail \ouralgo, as most evident on the multi-key task. In contrast, specialized architectures such as ALiBi, RNoPE-SWA, and NoPE underperform on multi-query tasks, which are the logic-intensive setting where strong base models excel. We believe these results provide compelling evidence toward validating \ouralgo's potential to be integrated as a standard component in the training pipeline of future generations of LMs.  

\begin{figure}[t]
    \vspace{-12pt}
    \centering
    \includegraphics[width=\linewidth]{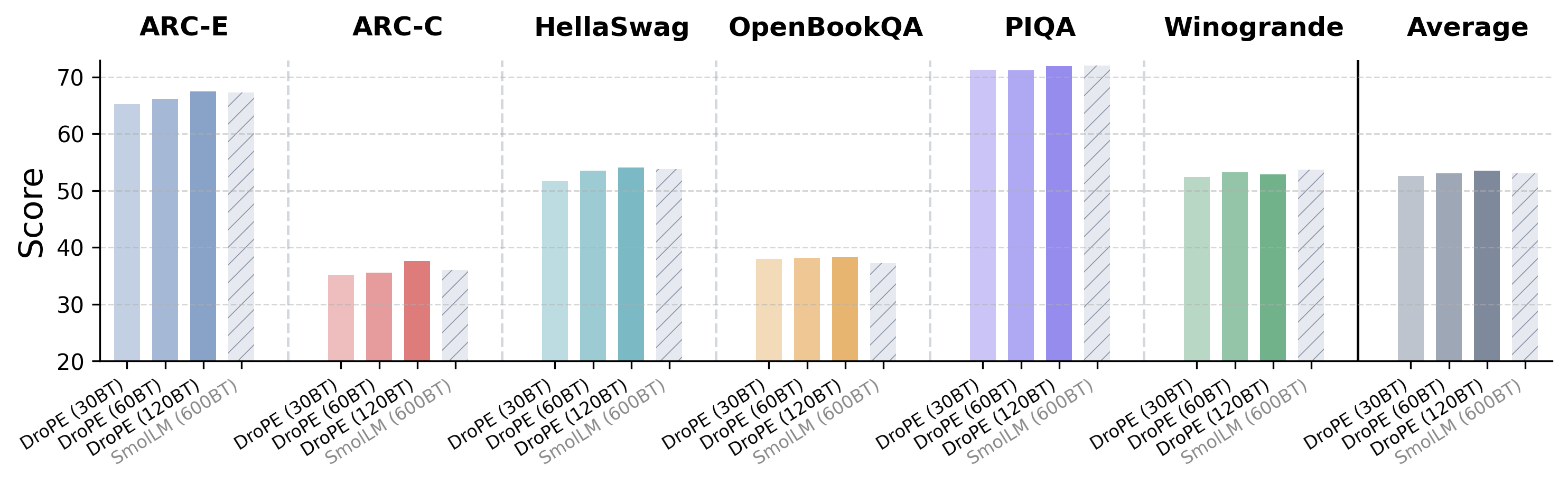}
    \vspace{-22pt}
    \caption{\textbf{\ouralgo matches base model in-context performance.} Comparison of base \textsc{SmolLM} with \textsc{SmolLM-\ouralgo} on standard LM benchmarks, using three recalibration recipes.} \label{fig:smoll_barplot}
    \vspace{-6pt}
\end{figure}

\textbf{Extending the context of LMs in the wild with \ouralgo.} For our second set of experiments, we directly apply \ouralgo to a 360M parameter language model from the \textsc{SmolLM} family~\citep{smollm} family pretrained on 600 billion tokens. We perform \ouralgo's recalibration with continued pretraining using the same context length, data, and hyperparameters as reported by \citet{smollm}. We consider three different recalibration budgets of 30, 60, and 120 billion tokens, adjusting the learning rate schedule accordingly. Given the extended training periods, only for these experiments, we also add QKNorm~\citep{qknorm} after dropping the positional embeddings, which we find beneficial for mitigating training instabilities, as noted by \citet{qknorm_olmo2} (See Appendix~\ref{apx:qknorm}).

\begin{table}[h]
    \caption{\textbf{\ouralgo outperforms RoPE-scaling methods on long context-tasks.} We evaluate \textsc{SmolLM-\ouralgo} and the base \textsc{SmolLM} model, extended with different RoPE scaling methods, on four long context language modeling tasks from~\citet{longbench} and needle-in-a-haystack.}\label{tab:longbench_result}
    \centering
    \small 
    \resizebox{0.95\textwidth}{!}{
        \begin{tabular}{lccccc|c}
            \toprule
            \textbf{Method} & \textbf{MultiFieldQA} & \textbf{MuSiQue} & \textbf{GovReport} & \textbf{LCC} & \textbf{NIAH} & \textbf{Avg.} \\
            \midrule
            \textsc{SmolLM} & $4.03$ & $0.4$ & $4.48$ & $5.99$ & $0.0$ & $2.98$\\
            \textsc{SmolLM} + PI & $13.68$ & $2.45$ & $5.67$ & $11.52$ & $0.0$ & $6.66$\\
            \textsc{SmolLM} + RoPE-NTK & $18.87$ & $4.89$ & $\mathbf{23.71}$ & $8.26$ & $29.84$ & $17.11$ \\
            \textsc{SmolLM} + YaRN & $20.78$ & $4.77$ & $15.03$ &  $10.87$ & $48.25$ & $19.94$ \\
            \midrule
            \textsc{SmolLM-\ouralgo} & $\mathbf{29.33}$ & $\mathbf{7.93}$ & $21.87$ &  $\mathbf{18.56}$ & $\mathbf{74.92}$ & $\mathbf{30.52}$ \\
            \bottomrule
        \end{tabular}
    }
\end{table}

We start by analyzing how quickly our \textsc{SmolLM-\ouralgo} models can recover \textsc{SmolLM}'s in-context performance across six different LM reasoning benchmarks~\citep{bench_1_arc, bench_2_hellaswag, bench_3_openbook_qa, bench_4_piqa, bench_5_winogrande}. As shown in Figures~\ref{fig:smoll_barplot} and \ref{fig:recalibration} as well as Table~\ref{tab:in_context_results}, even with our shortest training schedule, \textsc{SmolLM-\ouralgo} almost entirely matches \textsc{SmolLM} on every task, while with our longest schedule our new model manages to \textit{exceed} its original performance. Furthermore, inspecting our model at every checkpoint throughout training, we find that \ouralgo recovers \textit{over 95\% of \textsc{SmolLM}'s performance} after less than 5B tokens, representing a minuscule $0.8\%$ of \textsc{SmolLM}'s original budget.

\begin{wrapfigure}[15]{r}{0.45\textwidth}
    \centering
    \vspace{-23pt}
    \includegraphics[width=0.45\textwidth]{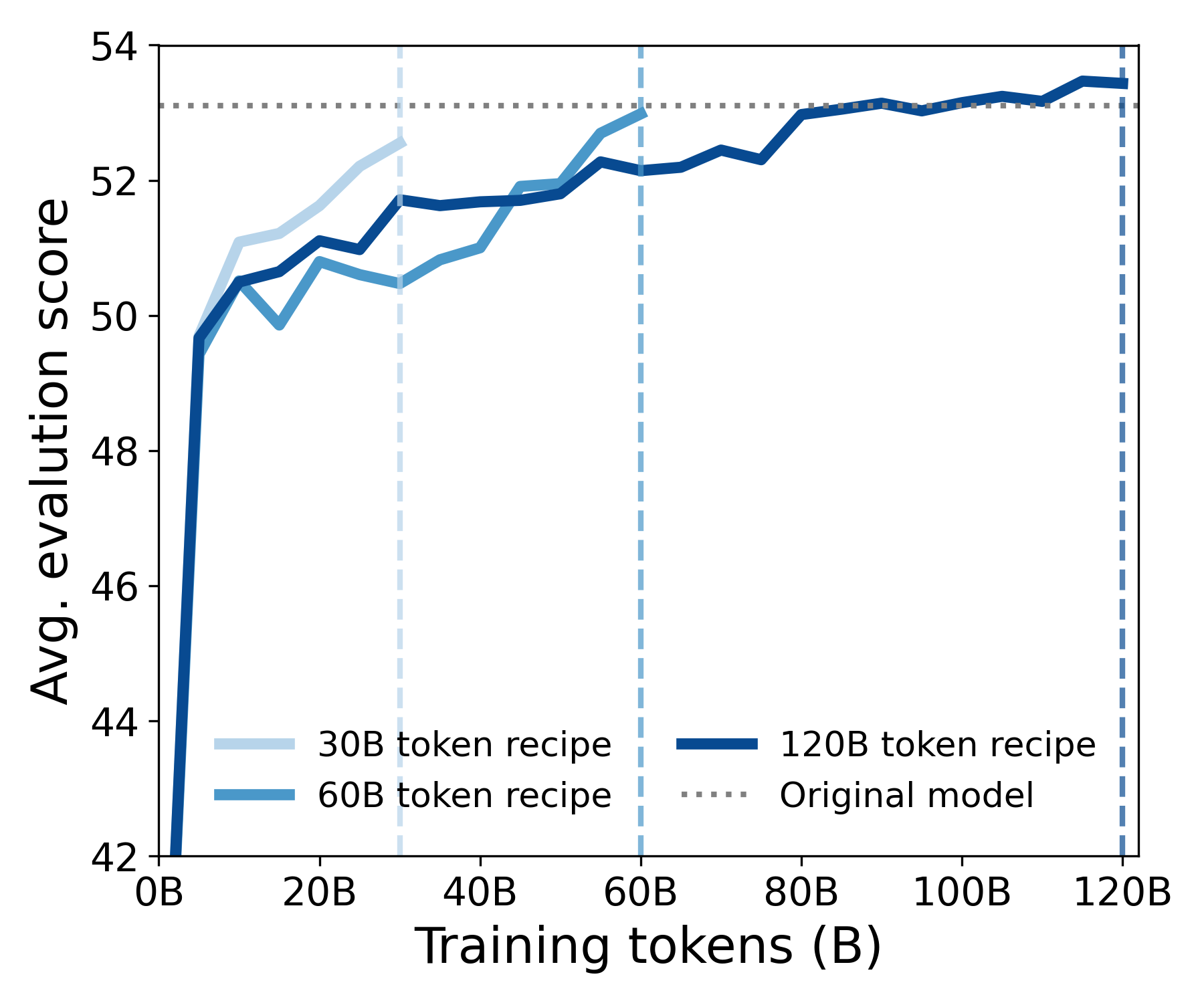}
    \vspace{-23pt}
    \caption{\textbf{\textsc{SmolLM-\ouralgo} recalibration.} We compare three recipes, using $30$B, $60$B, and $120$B training tokens.}\label{fig:recalibration} 
\end{wrapfigure}

We then evaluate our \textsc{SmolLM-\ouralgo} models' zero-shot length generalization on four different tasks from LongBench~\citep{longbench}, a challenging benchmark even for closed-source LMs, including knowledge-extraction problems longer than \emph{80 times} \textsc{SmolLM}'s pretraining context ($2048$ tokens). We compare our method with the base \textsc{SmolLM} and three RoPE extensions: PI, RoPE-NTK, and YaRN. As shown in Table~\ref{tab:longbench_result}, despite a significant difficulty spike compared to our prior evaluations, \ouralgo still displays a clear edge over prior approaches, improving the base \textsc{SmolLM}'s average score by over 10 times. These gains are far beyond all prior zero-shot RoPE extensions currently used across modern LMs. We refer to Appendix~\ref{apx:perf_v_exp} for a fine-grained analysis of task performance as a function of extension factor.

\textbf{Scaling to billion-parameter models.} Given the remarkable efficiency of recalibration, we test \ouralgo's ability to scale to larger LMs in the wild, such as \textsc{SmolLM-1.7B}~\citep{smollm} and \textsc{Llama2-7B}~\citep{llama2}, which were trained on 1 trillion and 4 trillion tokens, respectively. For both of these models, we perform recalibration on 20B tokens, which only represents 2\% of the pretraining budget for  \textsc{SmolLM-1.7B}, and only $0.5\%$ for \textsc{Llama2-7B}. As demonstrated in Table~\ref{tab:smollm_and_llama}, consistently with all our prior results on a smaller scale, \textsc{SmolLM-1.7B-\ouralgo} and \textsc{Llama2-7B-\ouralgo} once again outperform state-of-the-art RoPE-scaling methods on long-context question-answering and summarization, providing strong evidence towards the scalability and immediate potential of \ouralgo.

\begin{table}[h]
    \centering
    \small
    \caption{
    \looseness=-1
    \textbf{Length generalization results on larger models.} We evaluate \ouralgo on \textsc{SmolLM-1.7B} and \textsc{Llama2-7B}, and compare it against different RoPE scaling methods, on long context language modeling tasks from~\citet{longbench}. We don't include the LCC task since it is in context for \textsc{Llama2-7B}.}\label{tab:smollm_and_llama}
    \begin{tabular}{llccc|c}
        \toprule
        \textbf{Model} & \textbf{Method} & \textbf{MultiFieldQA} & \textbf{MuSiQue} & \textbf{GovReport} & \multicolumn{1}{c}{\textbf{Avg.}} \\
        \midrule
        \multirow{4}{*}{\textsc{SmolLM}-1.7B}
          & Base         &  4.12 & 0.50 &  4.70 &  3.11 \\
          & RoPE-NTK   & 27.58 & 3.37 & 24.65 & 18.53 \\
          & YaRN       & 27.60 & 3.90 & 17.19 & 16.23 \\
        \cmidrule(lr){2-6}
          & \ouralgo & \textbf{32.18} & \textbf{7.53} & \textbf{24.77} & \textbf{21.49} \\
        \midrule
        \multirow{4}{*}{\textsc{Llama2}-7B}
          & Base         & 17.26 & 10.43 & 32.41 & 20.03 \\
          & RoPE-NTK   & 21.81 & 10.91 & 32.91 & 21.88 \\
          & YaRN       & 23.13 &  7.65 & 26.65 & 19.14 \\
        \cmidrule(lr){2-6}
          & \ouralgo & \textbf{25.90} & \textbf{12.88} & \textbf{39.47} & \textbf{26.08} \\
        \bottomrule
    \end{tabular}
\end{table}

\looseness=-1
Overall, our in-context and out-of-context results demonstrate \ouralgo is an efficient and effective long-context extension method, which we believe can have meaningful implications for reducing the cost of training pipelines and for tackling the canonical context scalability challenges of transformers. We complement this section with additional experimental results, including the entire LongBench benchmark, and a performance by query length breakdown in Appendix~\ref{appendix: extra experiments}.
\section{Related work}\label{sec:6related}

Recent improvements to RoPE include variants based on Fourier and wavelet transforms \citep{fourier,wavelet} and methods such as $p$-RoPE~\citep{p_rope}, NRoPE-SWA~\citep{rope2nope}, and SWAN-GPT~\citep{swangpt}, which occupy a middle ground between RoPE and NoPE. Our approach represents a fundamentally different paradigm, replacing RoPE with NoPE at different stages of training. These directions are complementary to ours and can be used in place of RoPE within the DroPE framework. Another orthogonal direction seeks length generalization while retaining a dedicated positional vector yet modifying its indexing or adaptivity~\citep{decomposed_pv,middle_focused_pe,DAPE}.


\section{Discussion and extensions}\label{sec:7conclusion}
\looseness=-1
Our findings support a reinterpretation of positional embeddings in transformer LMs as a useful inductive bias that is essential for efficient training (Observation~\ref{find:pe_is_crucial}), but inherently constrains zero-shot context extension (Observation~\ref{find:yarn_is_bad_at_recall}). Based on these findings, we propose \ouralgo, a new method rethinking the conventional role of PEs as a temporary scaffold that can and should be removed after serving their training-time purpose (Observation~\ref{find:drope}). We empirically validate \ouralgo across different models and data scales, showing its effectiveness and potential to be integrated as a new core component of future state-of-the-art training pipelines. More broadly, our work demonstrates that canonical trade-offs in LM architecture design can be reconciled by employing different architectural choices for different stages of the training and inference pipelines and recalibrating the model for the new architecture. We hope this will inspire further research toward challenging established bottlenecks in AI.

\section*{Acknowledgments}
YG is supported by the UKRI Engineering and Physical Sciences Research Council (EPSRC) CDT in Autonomous and Intelligent Machines and Systems (grant reference EP/S024050/1).

\section*{Author contribution}

\textbf{Yoav Gelberg} led the project, made major contributions to the codebase, method design, training pipelines, experimentation, and writing. He led \ouralgo's empirical and theoretical analysis.

\textbf{Koshi Eguchi} made contributions to experimentation and provided infrastructure support for scaling the methodology.

\textbf{Takuya Akiba} advised the project, was involved in early discussions about method design, and contributed to writing.

\textbf{Edoardo Cetin} made major contributions to the codebase, method design, experimentation, and writing. He led engineering on DroPE's training pipeline and coordinated the project.

\newpage
\bibliography{references}

@article{nope,
  title={The impact of positional encoding on length generalization in transformers},
  author={Kazemnejad, Amirhossein and Padhi, Inkit and Natesan Ramamurthy, Karthikeyan and Das, Payel and Reddy, Siva},
  journal={Advances in Neural Information Processing Systems},
  volume={36},
  pages={24892--24928},
  year={2023}
}

@article{rope,
  title={Roformer: Enhanced transformer with rotary position embedding},
  author={Su, Jianlin and Ahmed, Murtadha and Lu, Yu and Pan, Shengfeng and Bo, Wen and Liu, Yunfeng},
  journal={Neurocomputing},
  volume={568},
  pages={127063},
  year={2024},
  publisher={Elsevier}
}

@article{vaswani2017attention,
  title={Attention is all you need},
  author={Vaswani, Ashish and Shazeer, Noam and Parmar, Niki and Uszkoreit, Jakob and Jones, Llion and Gomez, Aidan N and Kaiser, {\L}ukasz and Polosukhin, Illia},
  journal={Advances in neural information processing systems},
  volume={30},
  year={2017}
}

@article{barbero2024round,
  title={Round and round we go! what makes rotary positional encodings useful?},
  author={Barbero, Federico and Vitvitskyi, Alex and Perivolaropoulos, Christos and Pascanu, Razvan and Veli{\v{c}}kovi{\'c}, Petar},
  journal={arXiv preprint arXiv:2410.06205},
  year={2024}
}

@misc{controlledstudylongcontext_rush,
      title={A Controlled Study on Long Context Extension and Generalization in LLMs}, 
      author={Yi Lu and Jing Nathan Yan and Songlin Yang and Justin T. Chiu and Siyu Ren and Fei Yuan and Wenting Zhao and Zhiyong Wu and Alexander M. Rush},
      year={2024},
      eprint={2409.12181},
      archivePrefix={arXiv},
      primaryClass={cs.CL},
      url={https://arxiv.org/abs/2409.12181}, 
}

@article{haviv2022transformer,
  title={Transformer language models without positional encodings still learn positional information},
  author={Haviv, Adi and Ram, Ori and Press, Ofir and Izsak, Peter and Levy, Omer},
  journal={arXiv preprint arXiv:2203.16634},
  year={2022}
}

@article{pi,
  title={Extending context window of large language models via positional interpolation},
  author={Chen, Shouyuan and Wong, Sherman and Chen, Liangjian and Tian, Yuandong},
  journal={arXiv preprint arXiv:2306.15595},
  year={2023}
}

@article{yarn,
  title={Yarn: Efficient context window extension of large language models},
  author={Peng, Bowen and Quesnelle, Jeffrey and Fan, Honglu and Shippole, Enrico},
  journal={arXiv preprint arXiv:2309.00071},
  year={2023}
}

@article{rope2nope,
  title={Rope to nope and back again: A new hybrid attention strategy},
  author={Yang, Bowen and Venkitesh, Bharat and Talupuru, Dwarak and Lin, Hangyu and Cairuz, David and Blunsom, Phil and Locatelli, Acyr},
  journal={arXiv preprint arXiv:2501.18795},
  year={2025}
}

@article{lost_in_the_middle,
  title={Lost in the middle: How language models use long contexts},
  author={Liu, Nelson F and Lin, Kevin and Hewitt, John and Paranjape, Ashwin and Bevilacqua, Michele and Petroni, Fabio and Liang, Percy},
  journal={arXiv preprint arXiv:2307.03172},
  year={2023}
}

@article{gemini,
  title={Gemini: a family of highly capable multimodal models},
  author={Team, Gemini and Anil, Rohan and Borgeaud, Sebastian and Alayrac, Jean-Baptiste and Yu, Jiahui and Soricut, Radu and Schalkwyk, Johan and Dai, Andrew M and Hauth, Anja and Millican, Katie and others},
  journal={arXiv preprint arXiv:2312.11805},
  year={2023}
}

@article{gpt3,
  title={Language models are few-shot learners},
  author={Brown, Tom and Mann, Benjamin and Ryder, Nick and Subbiah, Melanie and Kaplan, Jared D and Dhariwal, Prafulla and Neelakantan, Arvind and Shyam, Pranav and Sastry, Girish and Askell, Amanda and others},
  journal={Advances in neural information processing systems},
  volume={33},
  pages={1877--1901},
  year={2020}
}

@article{alphafold,
  title={Highly accurate protein structure prediction with AlphaFold},
  author={Jumper, John and Evans, Richard and Pritzel, Alexander and Green, Tim and Figurnov, Michael and Ronneberger, Olaf and Tunyasuvunakool, Kathryn and Bates, Russ and {\v{Z}}{\'\i}dek, Augustin and Potapenko, Anna and others},
  journal={nature},
  volume={596},
  number={7873},
  pages={583--589},
  year={2021},
  publisher={Nature Publishing Group UK London}
}

@article{vit,
  title={An image is worth 16x16 words: Transformers for image recognition at scale},
  author={Dosovitskiy, Alexey and Beyer, Lucas and Kolesnikov, Alexander and Weissenborn, Dirk and Zhai, Xiaohua and Unterthiner, Thomas and Dehghani, Mostafa and Minderer, Matthias and Heigold, Georg and Gelly, Sylvain and others},
  journal={arXiv preprint arXiv:2010.11929},
  year={2020}
}

@article{flashattention,
  title={Flashattention: Fast and memory-efficient exact attention with io-awareness},
  author={Dao, Tri and Fu, Dan and Ermon, Stefano and Rudra, Atri and R{\'e}, Christopher},
  journal={Advances in neural information processing systems},
  volume={35},
  pages={16344--16359},
  year={2022}
}

@article{ringattention,
  title={Ring attention with blockwise transformers for near-infinite context},
  author={Liu, Hao and Zaharia, Matei and Abbeel, Pieter},
  journal={arXiv preprint arXiv:2310.01889},
  year={2023}
}

@article{blockattention,
  title={Blockwise parallel transformers for large context models},
  author={Liu, Hao and Abbeel, Pieter},
  journal={Advances in neural information processing systems},
  volume={36},
  pages={8828--8844},
  year={2023}
}

@article{longrope,
  title={Longrope: Extending llm context window beyond 2 million tokens},
  author={Ding, Yiran and Zhang, Li Lyna and Zhang, Chengruidong and Xu, Yuanyuan and Shang, Ning and Xu, Jiahang and Yang, Fan and Yang, Mao},
  journal={arXiv preprint arXiv:2402.13753},
  year={2024}
}

@article{alibi,
  title={Train short, test long: Attention with linear biases enables input length extrapolation},
  author={Press, Ofir and Smith, Noah A and Lewis, Mike},
  journal={arXiv preprint arXiv:2108.12409},
  year={2021}
}

@article{chi2023attention,
  title={Attention alignment and flexible positional embeddings improve transformer length extrapolation},
  author={Chi, Ta-Chung and Fan, Ting-Han and Rudnicky, Alexander I},
  journal={arXiv preprint arXiv:2311.00684},
  year={2023}
}

@article{acontrolledstudy,
  title={A controlled study on long context extension and generalization in llms},
  author={Lu, Yi and Yan, Jing Nathan and Yang, Songlin and Chiu, Justin T and Ren, Siyu and Yuan, Fei and Zhao, Wenting and Wu, Zhiyong and Rush, Alexander M},
  journal={arXiv preprint arXiv:2409.12181},
  year={2024}
}

@article{swangpt,
  title={Swan-gpt: An efficient and scalable approach for long-context language modeling},
  author={Puvvada, Krishna C and Ladhak, Faisal and Serrano, Santiago Akle and Hsieh, Cheng-Ping and Acharya, Shantanu and Majumdar, Somshubra and Jia, Fei and Kriman, Samuel and Sun, Simeng and Rekesh, Dima and others},
  journal={arXiv preprint arXiv:2504.08719},
  year={2025}
}

@article{performer,
  title={Rethinking attention with performers},
  author={Choromanski, Krzysztof and Likhosherstov, Valerii and Dohan, David and Song, Xingyou and Gane, Andreea and Sarlos, Tamas and Hawkins, Peter and Davis, Jared and Mohiuddin, Afroz and Kaiser, Lukasz and others},
  journal={arXiv preprint arXiv:2009.14794},
  year={2020}
}

@misc{ntk,
  author       = {bloc97},
  title        = {NTK-Aware Scaled RoPE allows LLaMA models to have extended (8k+) context size without any fine-tuning and minimal perplexity degradation},
  howpublished = {Reddit post on r/LocalLLaMA},
  year         = {2023},
  month        = jun,
  url          = {https://www.reddit.com/r/LocalLLaMA/comments/14lz7j5/ntkaware_scaled_rope_allows_llama_models_to_have/},
  note         = {Accessed: 2025-09-19}
}

@article{ruler,
  title   = {RULER: What's the Real Context Size of Your Long-Context Language Models?},
  author  = {Cheng-Ping Hsieh and Simeng Sun and Samuel Kriman and Shantanu Acharya and Dima Rekesh and Fei Jia and Yang Zhang and Boris Ginsburg},
  journal = {arXiv preprint arXiv:2404.06654},
  year    = {2024},
  doi     = {10.48550/arXiv.2404.06654},
  url     = {https://arxiv.org/abs/2404.06654}
}

@article{longbench,
  title   = {LongBench: A Bilingual, Multitask Benchmark for Long Context Understanding},
  author  = {Yushi Bai and Xin Lv and Jiajie Zhang and Hongchang Lyu and Jiankai Tang and Zhidian Huang and Zhengxiao Du and Xiao Liu and Aohan Zeng and Lei Hou and Yuxiao Dong and Jie Tang and Juanzi Li},
  journal = {arXiv preprint arXiv:2308.14508},
  year    = {2023},
  doi     = {10.48550/arXiv.2308.14508},
  url     = {https://arxiv.org/abs/2308.14508}
}

@misc{niah_repo,
  author       = {Greg Kamradt},
  title        = {Needle In A Haystack — Pressure Testing LLMs},
  year         = {2023},
  howpublished = {\url{https://github.com/gkamradt/LLMTest_NeedleInAHaystack}},
  note         = {GitHub repository}
}

@book{vershynin2018high,
  title={High-dimensional probability: An introduction with applications in data science},
  author={Vershynin, Roman},
  volume={47},
  year={2018},
  publisher={Cambridge university press}
}

@misc{smollm,
      title={SmolLM - blazingly fast and remarkably powerful}, 
      author={Loubna Ben Allal and Anton Lozhkov and Elie Bakouch and Leandro von Werra and Thomas Wolf},
      year={2024},
}

@article{fineweb,
  title={The fineweb datasets: Decanting the web for the finest text data at scale},
  author={Penedo, Guilherme and Kydl{\'\i}{\v{c}}ek, Hynek and Lozhkov, Anton and Mitchell, Margaret and Raffel, Colin A and Von Werra, Leandro and Wolf, Thomas and others},
  journal={Advances in Neural Information Processing Systems},
  volume={37},
  pages={30811--30849},
  year={2024}
}

@software{smollm_corpus,
  author = {Ben Allal, Loubna and Lozhkov, Anton and Penedo, Guilherme and Wolf, Thomas and von Werra, Leandro},
  title = {SmolLM-Corpus},
  month = July,
  year = 2024,
  url = {https://huggingface.co/datasets/HuggingFaceTB/smollm-corpus}
}

@article{linformer,
  title={Linformer: Self-attention with linear complexity},
  author={Wang, Sinong and Li, Belinda Z and Khabsa, Madian and Fang, Han and Ma, Hao},
  journal={arXiv preprint arXiv:2006.04768},
  year={2020}
}

@inproceedings{nystromformer,
  title={Nystr{\"o}mformer: A nystr{\"o}m-based algorithm for approximating self-attention},
  author={Xiong, Yunyang and Zeng, Zhanpeng and Chakraborty, Rudrasis and Tan, Mingxing and Fung, Glenn and Li, Yin and Singh, Vikas},
  booktitle={Proceedings of the AAAI conference on artificial intelligence},
  volume={35},
  number={16},
  pages={14138--14148},
  year={2021}
}

@article{bigbird,
  title={Big bird: Transformers for longer sequences},
  author={Zaheer, Manzil and Guruganesh, Guru and Dubey, Kumar Avinava and Ainslie, Joshua and Alberti, Chris and Ontanon, Santiago and Pham, Philip and Ravula, Anirudh and Wang, Qifan and Yang, Li and others},
  journal={Advances in neural information processing systems},
  volume={33},
  pages={17283--17297},
  year={2020}
}

@book{optimization,
  title={Introductory lectures on convex optimization: A basic course},
  author={Nesterov, Yurii},
  volume={87},
  year={2013},
  publisher={Springer Science \& Business Media}
}

@article{chinchilla,
  title={Training compute-optimal large language models},
  author={Hoffmann, Jordan and Borgeaud, Sebastian and Mensch, Arthur and Buchatskaya, Elena and Cai, Trevor and Rutherford, Eliza and Casas, Diego de Las and Hendricks, Lisa Anne and Welbl, Johannes and Clark, Aidan and others},
  journal={arXiv preprint arXiv:2203.15556},
  year={2022}
}

@article{qwen2,
  title={Qwen2 technical report, 2024},
  author={Yang, An and Yang, Baosong and Hui, Binyuan and Zheng, Bo and Yu, Bowen and Zhou, Chang and Li, Chengpeng and Li, Chengyuan and Liu, Dayiheng and Huang, Fei and others},
  journal={URL https://arxiv. org/abs/2407.10671},
  volume={7},
  pages={8},
  year={2024}
}

@article{adamw,
  title={Decoupled weight decay regularization},
  author={Loshchilov, Ilya and Hutter, Frank},
  journal={arXiv preprint arXiv:1711.05101},
  year={2017}
}

@article{gpt2,
  title={Language models are unsupervised multitask learners},
  author={Radford, Alec and Wu, Jeffrey and Child, Rewon and Luan, David and Amodei, Dario and Sutskever, Ilya and others},
  journal={OpenAI blog},
  volume={1},
  number={8},
  pages={9},
  year={2019}
}

@article{llama2,
  title={Llama 2: Open foundation and fine-tuned chat models},
  author={Touvron, Hugo and Martin, Louis and Stone, Kevin and Albert, Peter and Almahairi, Amjad and Babaei, Yasmine and Bashlykov, Nikolay and Batra, Soumya and Bhargava, Prajjwal and Bhosale, Shruti and others},
  journal={arXiv preprint arXiv:2307.09288},
  year={2023}
}

@article{qknorm,
  title={Query-key normalization for transformers},
  author={Henry, Alex and Dachapally, Prudhvi Raj and Pawar, Shubham and Chen, Yuxuan},
  journal={arXiv preprint arXiv:2010.04245},
  year={2020}
}

@article{wang2024length,
  title={Length generalization of causal transformers without position encoding},
  author={Wang, Jie and Ji, Tao and Wu, Yuanbin and Yan, Hang and Gui, Tao and Zhang, Qi and Huang, Xuanjing and Wang, Xiaoling},
  journal={arXiv preprint arXiv:2404.12224},
  year={2024}
}

@inproceedings{
p_rope,
title={Round and Round We Go! What makes Rotary Positional Encodings useful?},
author={Federico Barbero and Alex Vitvitskyi and Christos Perivolaropoulos and Razvan Pascanu and Petar Veli{\v{c}}kovi{\'c}},
booktitle={The Thirteenth International Conference on Learning Representations},
year={2025},
url={https://openreview.net/forum?id=GtvuNrk58a}
}

@inproceedings{
wavelet,
title={Wavelet-based Positional Representation for Long Context},
author={Yui Oka and Taku Hasegawa and Kyosuke Nishida and Kuniko Saito},
booktitle={The Thirteenth International Conference on Learning Representations},
year={2025},
url={https://openreview.net/forum?id=OhauMUNW8T}
}

@inproceedings{
fourier,
title={Fourier Position Embedding: Enhancing Attention{\textquoteright}s Periodic Extension for Length Generalization},
author={Ermo Hua and Che Jiang and Xingtai Lv and Kaiyan Zhang and Youbang Sun and Yuchen Fan and Xuekai Zhu and Biqing Qi and Ning Ding and Bowen Zhou},
booktitle={Forty-second International Conference on Machine Learning},
year={2025},
url={https://openreview.net/forum?id=ZfDNDkg7Dh}
}

@inproceedings{
decomposed_pv,
title={Exploring Context Window of Large Language Models via Decomposed Positional Vectors},
author={zican Dong and Junyi Li and Xin Men and Xin Zhao and Bingning Wang and Zhen Tian and weipeng chen and Ji-Rong Wen},
booktitle={The Thirty-eighth Annual Conference on Neural Information Processing Systems},
year={2024},
url={https://openreview.net/forum?id=zeYyq0GpXO}
}

@inproceedings{
middle_focused_pe,
title={An Efficient Recipe for Long Context Extension via Middle-Focused Positional Encoding},
author={Tong Wu and Yanpeng Zhao and Zilong Zheng},
booktitle={The Thirty-eighth Annual Conference on Neural Information Processing Systems},
year={2024},
url={https://openreview.net/forum?id=aNHEqFMS0N}
}

@inproceedings{
DAPE,
title={{DAPE}: Data-Adaptive Positional Encoding for Length Extrapolation},
author={Chuanyang Zheng and Yihang Gao and Han Shi and Minbin Huang and Jingyao Li and Jing Xiong and Xiaozhe Ren and Michael Ng and Xin Jiang and Zhenguo Li and Yu Li},
booktitle={The Thirty-eighth Annual Conference on Neural Information Processing Systems},
year={2024},
url={https://openreview.net/forum?id=rnUEUbRxVu}
}

@article{bench_1_arc,
  title={Think you have solved question answering? try arc, the ai2 reasoning challenge},
  author={Clark, Peter and Cowhey, Isaac and Etzioni, Oren and Khot, Tushar and Sabharwal, Ashish and Schoenick, Carissa and Tafjord, Oyvind},
  journal={arXiv preprint arXiv:1803.05457},
  year={2018}
}

@article{bench_2_hellaswag,
  title={Hellaswag: Can a machine really finish your sentence?},
  author={Zellers, Rowan and Holtzman, Ari and Bisk, Yonatan and Farhadi, Ali and Choi, Yejin},
  journal={arXiv preprint arXiv:1905.07830},
  year={2019}
}

@article{bench_3_openbook_qa,
  title={Can a suit of armor conduct electricity? a new dataset for open book question answering},
  author={Mihaylov, Todor and Clark, Peter and Khot, Tushar and Sabharwal, Ashish},
  journal={arXiv preprint arXiv:1809.02789},
  year={2018}
}

@inproceedings{bench_4_piqa,
  title={Piqa: Reasoning about physical commonsense in natural language},
  author={Bisk, Yonatan and Zellers, Rowan and Gao, Jianfeng and Choi, Yejin and others},
  booktitle={Proceedings of the AAAI conference on artificial intelligence},
  volume={34},
  number={05},
  pages={7432--7439},
  year={2020}
}

@article{bench_5_winogrande,
  title={Winogrande: An adversarial winograd schema challenge at scale},
  author={Sakaguchi, Keisuke and Bras, Ronan Le and Bhagavatula, Chandra and Choi, Yejin},
  journal={Communications of the ACM},
  volume={64},
  number={9},
  pages={99--106},
  year={2021},
  publisher={ACM New York, NY, USA}
}

@misc{lighteval,
  author  = {Habib, Nathan and Fourrier, Cl{\'e}mentine and Kydl{\'\i}{\v{c}}ek, Hynek and Wolf, Thomas and Tunstall, Lewis},
  title   = {LightEval: A lightweight framework for LLM evaluation},
  year    = {2023},
  version = {0.11.0},
  url     = {https://github.com/huggingface/lighteval}
}

@article{qknorm_olmo2,
  title={2 OLMo 2 Furious},
  author={OLMo, Team and Walsh, Pete and Soldaini, Luca and Groeneveld, Dirk and Lo, Kyle and Arora, Shane and Bhagia, Akshita and Gu, Yuling and Huang, Shengyi and Jordan, Matt and others},
  journal={arXiv preprint arXiv:2501.00656},
  year={2024}
}

@article{olmo2,
  title={2 OLMo 2 Furious},
  author={OLMo, Team and Walsh, Pete and Soldaini, Luca and Groeneveld, Dirk and Lo, Kyle and Arora, Shane and Bhagia, Akshita and Gu, Yuling and Huang, Shengyi and Jordan, Matt and others},
  journal={arXiv preprint arXiv:2501.00656},
  year={2024}
}

@article{qwen3,
  title={Qwen3 technical report},
  author={Yang, An and Li, Anfeng and Yang, Baosong and Zhang, Beichen and Hui, Binyuan and Zheng, Bo and Yu, Bowen and Gao, Chang and Huang, Chengen and Lv, Chenxu and others},
  journal={arXiv preprint arXiv:2505.09388},
  year={2025}
}

@article{longrope2,
  title={LongRoPE2: Near-Lossless LLM Context Window Scaling},
  author={Shang, Ning and Zhang, Li Lyna and Wang, Siyuan and Zhang, Gaokai and Lopez, Gilsinia and Yang, Fan and Chen, Weizhu and Yang, Mao},
  journal={arXiv preprint arXiv:2502.20082},
  year={2025}
}
\bibliographystyle{plainnat}

\clearpage
\appendix

\section{Extended preliminaries}\label{apx:extended_preliminaries}
\paragraph{Attention.}
Throughout this section, we consider a pre-norm, decoder-only transformer with $L$ layers, $H$ attention heads per layer, model dimension $d = d_\mathrm{model}$, and head dimension $d_k$. $h^{(l)}_1,\dots,h^{(l)}_T\in\sR^{d}$ denote the representations fed into the $l$-th multi-head attention block. For a head $h$ in layer $l$, queries, keys, and values are computed by 
\begin{equation}
    q_i^{(l,h)} = W_Q^{(l,h)} h_i^{(l)}, \qquad k_i^{(l,h)} = W_K^{(l,h)} h_i^{(l)}, \qquad v_i^{(l,h)} = W_V^{(l,h)} h_i^{(l)}, 
\end{equation}
The attention scores and weights are then computed by
\begin{equation}\label{eq:attention_apx}
    s_{ij}^{(l,h)} = \frac{1}{\sqrt{d_k}} \big(q_i^{(l,h)}\big)^\top k_j^{(l,h)}, \qquad
    \alpha_{ij}^{(l,h)} = \softmax(s_{i1}^{(l,h)}, \dots, s_{ii}^{(l,h)})_j.
\end{equation}
$s_{ij}^{(l,h)}$ are referred to as attention \emph{logits} or \emph{scores} and $\alpha_{ij}^{(l,h)}$ are referred to as attention \emph{weights} or \emph{probabilities}. Note that the softmax is taken over $j \leq i$, implementing a causal mask. The output of the multi-head attention block is 
\begin{equation}
    z_i^{(l, h)} = \sum_{j \leq i} \alpha_{ij}^{(l, h)} v_j^{(l, h)}, \qquad o_i^{(l)} = W_O^{(l)}[z_i^{(l, 1)}, \dots, z_i^{(l, H)}],
\end{equation}
where $[\cdot, \dots, \cdot]$ represents concatenation along the feature dimension. When clear from context, we omit layer and head indices. 

\paragraph{Positional embeddings in transformers.}
The attention mechanism does not directly encode relative distances between queries and keys. Therefore, attention is invariant to prefix permutations: for any permutation $\sigma \in S_p$ of the first $p$ input tokens, $\mathrm{attn}(x_{\sigma^{-1}(1)}, \dots, x_{\sigma^{-1}(p)}, x_p, \dots, x_T)_i = \mathrm{attn}(x_1, \dots, x_T)_i$ for every $i > p$. In other words, pure attention is \emph{blind} to token positions. To address this, \citet{vaswani2017attention} introduced \emph{absolute} positional embeddings, adding position information to the token embeddings before the first transformer block. More recently, many architectures replace absolute embeddings with \emph{relative} schemes that inject pairwise positional information directly into the attention mechanism. The most widely used approach is Rotary Position Embedding (RoPE) \citep{rope}. RoPE modifies the attention scores in Equation~\ref{eq:attention_apx} by rotating queries and keys before taking their inner product:
\begin{equation}
    s_{ij}^{\mathrm{RoPE}} = \frac{1}{\sqrt{d_k}}q_i^\top R^{j - i} k_j, \qquad \alpha_{ij}^\mathrm{RoPE} = \softmax(s_{i1}^\mathrm{RoPE}, \cdots, s_{ii}^\mathrm{RoPE})_j,
\end{equation}
where, $R \in O(d_k)$ is a block-diagonal orthogonal matrix composed out of $2 \times 2$ rotation blocks: 
\begin{equation}
    R=\operatorname{block\text{-}diag} \big(R(\omega_1),\ldots,R(\omega_{d_k/2})\big), \quad
    R(\omega)=
    \begin{pmatrix}
        \cos(\omega) & -\sin(\omega) \\
        \sin(\omega) & \cos(\omega)
    \end{pmatrix}.
\end{equation}
In the standard RoPE parameterization, $\omega_m = b^{-2\frac{m-1}{d_k}}$ with $b = 10{,}000$. 


\paragraph{Language model context extension.}
Generalizing to contexts longer than those seen during training is a key challenge for transformer-based language models. The key issue is that when applying a transformer on a longer context, the attention mechanism must operate over more tokens than it was trained to handle. This issue is exacerbated with RoPE: applying RoPE to sequences beyond the training length introduces larger position deltas, and thus larger rotations, pushing attention logits out of the training distribution. RoPE context-extension methods address this by \emph{rescaling the RoPE frequencies} when the inference context length exceeds the training context length. Let $C_\mathrm{train}$ be the training context and $C_\mathrm{test}>C_\mathrm{train}$ the target context with extension factor $s= C_\mathrm{test} / C_\mathrm{train}$. Such methods define new frequencies 
\begin{equation*}
    \omega_m' = \gamma_m \omega_m, \qquad m=1, \dots, \tfrac{d_k}{2},
\end{equation*}
using \emph{scaling factors} $\gamma_m = \gamma_m(s)$. E.g. Position Interpolation (PI)~\citep{pi}, uses a uniform scaling of
\begin{equation}\label{eq:pi}
    \gamma_m^\mathrm{PI} = \tfrac{1}{s}.
\end{equation}
NTK-RoPE~\citep{ntk} uses
\begin{equation}\label{eq:ntk}
    \gamma_m^\mathrm{NTK} = \left(\tfrac{1}{s}\right)^{\frac{2m}{d_k -2}}, 
\end{equation}
so that low frequencies ($m \approx d_k/2$) are scaled similarly to PI and for high frequencies $\gamma_m \approx 1$. YaRN~\citep{yarn} uses
\begin{equation}\label{eq:yarn}
    \gamma_m^\mathrm{YaRN} = (1-\kappa_m) \tfrac{1}{s} + \kappa_m, \qquad \kappa_m = \begin{cases}
        0 & \omega_m < p \\
        1 & \omega_m > q \\
        \tfrac{\omega_m-p}{q -p} & p \leq \omega_m \leq q,
    \end{cases}
\end{equation}
with tunable $p$ and $q$ parameters, originally chosen as $p = 1$, $q = 32$. See Figure~\ref{fig:freq_scaling_ext} for a comparison between these different RoPE scaling methods with $s = 2$, $3$, and $4$.
\begin{figure}
    \centering
    \includegraphics[width=0.8\linewidth]{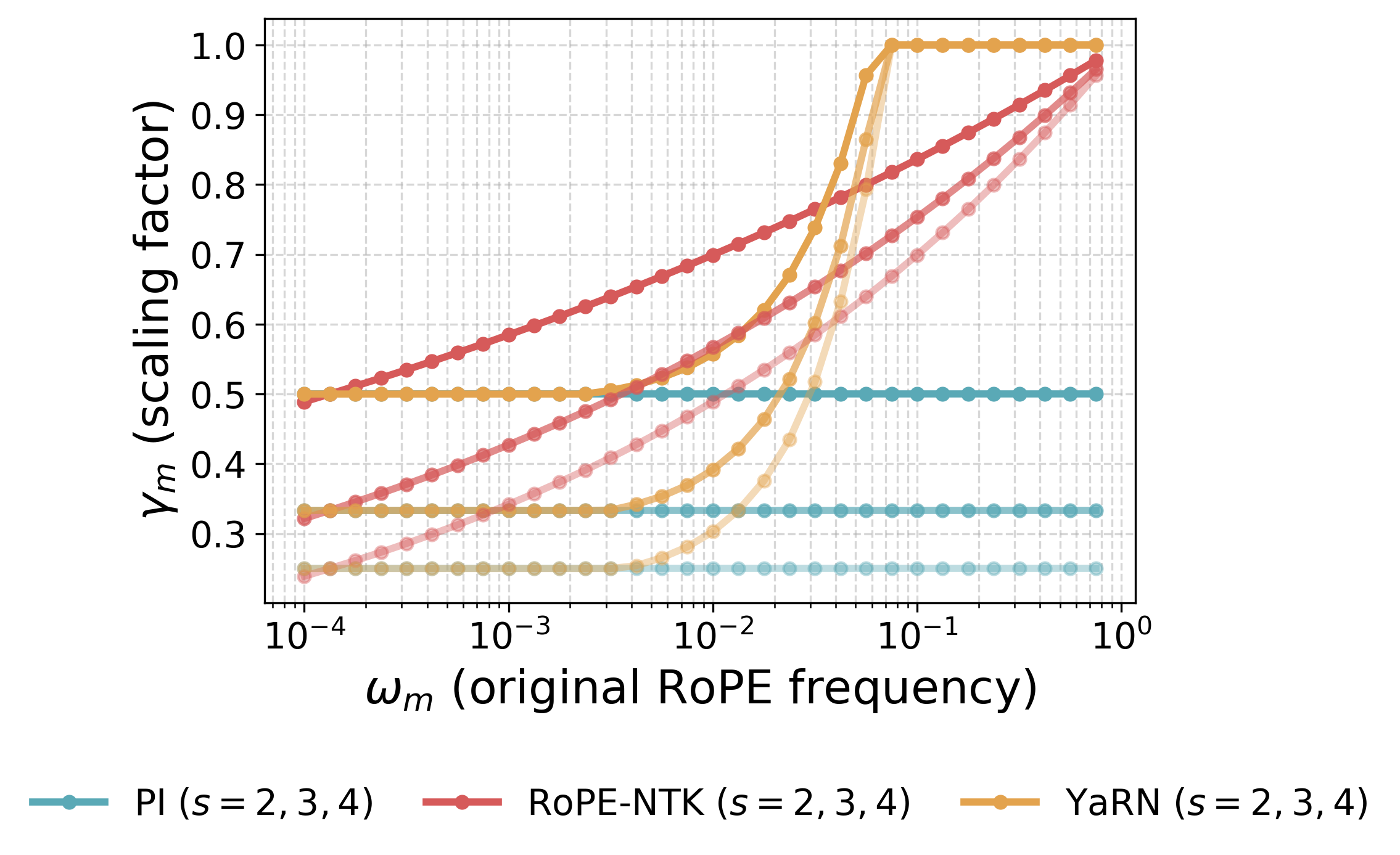}
    \caption{RoPE frequency scaling under PI, NTK-aware scaling (RoPE-NTK), and YaRN, with scaling factors $s = 2, 3, 4$.}\label{fig:freq_scaling_ext}
\end{figure}

\section{Theoretical results and proofs}\label{appendix:extended_theory}
In this section, we analyze the behavior of positional bias, or attention non-uniformity, in NoPE transformers and RoPE transformers early in training. We provide formal statements and proofs for all the results from Section~\ref{sec:pe_is_important}, starting with Propositions~\ref{prop:const_seq} and~\ref{prop:rope_lowerbound}, followed by Theorem~\ref{thm:propagation}. The notation of this section follows that of Appendix~\ref{apx:extended_preliminaries}.

\subsection{Proof of Proposition~\ref{prop:const_seq}}
\ConstSeq*
\begin{proof}
    Let $x_1, \dots, x_T$ be a constant input sequence, $x_1 = \cdots = x_T$, and let $\mathsf{M}$ be a NoPE transformer, i.e. a transformer with \emph{no positional encodings} and causal self attention. The order of the proof is (4) $\Rightarrow$ (1) $\Rightarrow$ (2 + 3).
    
    \textbf{(4) Layer outputs, and thus model outputs, are constant.}
    At the first layer, inputs are identical $h_1^{(1)} = \cdots = h_L^{(1)} = h$. This means that for every attention head and every $1 \leq j \leq T$
    \begin{equation*}
        v_j \equiv v = W_V h.
    \end{equation*}
    Therefore, the output of the attention head is 
    \begin{equation*}
        z_i = \sum_{j\le i}\alpha_{ij} v_j =  \sum_{j\le i}\alpha_{ij} v = \Big(\sum_{j\le i}\alpha_{ij}\Big) v = 1\cdot v,
    \end{equation*}
    independent of $i$. Concatenating heads and applying $W_O$ preserves equality across positions. Residual connections, LayerNorm, and the MLP are positionwise (the same function is applied independently at each position), so identical inputs produce identical outputs at every position. Thus the layer output remains constant. By repeating this argument layer-by-layer, every subsequent layer receives identical inputs and outputs identical states, so in the end
    \begin{equation*}
        M(x)_1=\cdots=M(x)_L.
    \end{equation*}
    
    \textbf{(1) Uniform causal attention.}
    Using (4), we know that for every layer $1 \leq l \leq L$ 
    \begin{equation*}
        h_1^{(l)} = \cdots = h_L^{(l)} = h.
    \end{equation*}
    Therefore, for every attention head and every $1 \leq j \leq T$
    \begin{equation*}
        q_j \equiv q := W_Q h, \quad k_j \equiv k := W_K h, \quad v_j \equiv v : =W_V h.
    \end{equation*}
    Thus, for each $1 \leq j \leq i \leq T$, the attention scores $s_{ij} = q^\top k/\sqrt{d_k} \equiv c$ are constant (independent of $i$ or $j$). Hence
    \begin{equation*}
        \alpha_{ij} = \softmax(\underbrace{c, \dots, c}_{i\ \text{entries}})_j=\tfrac{1}{i}\qquad(j\le i).
    \end{equation*}
    
    \textbf{(2 + 3) Vanishing $W_Q,W_K$ gradients.}
    Since, the inputs for every layer are constant, we know from (1) that every attention head has $\alpha_{ij} \equiv 1/i$, \emph{independant of $W_Q$ and $W_K$}. Therefore $\partial \alpha_{ij} / \partial W_Q = \partial \alpha_{ij} / \partial W_K = 0$. Since the attention bias $\abias^c$ depends on the parameters $\params$ only through $\alpha_{ij}$ and the loss $\gL$ depends on $W_Q$ and $W_K$ only through $\alpha_{ij}$, all these gradients vanish. More formally, using the chain rule, 
    \begin{gather*}
        \frac{\partial \abias^c}{\partial \params} = \frac{1}{T}\sum_{1 \leq j \leq i \leq T} c_{ij} \frac{\partial \alpha_{ij}}{\partial \params} = 0, \\
        \frac{\partial \gL}{\partial W_Q} = \sum_{1 \leq j \leq i \leq T} \frac{\partial \gL}{\partial \alpha_{ij}} \frac{\partial \alpha_{ij}}{\partial W_Q} = 0, \qquad
        \frac{\partial \gL}{\partial W_K} = \sum_{1 \leq j \leq i \leq T} \frac{\partial \gL}{\partial \alpha_{ij}} \frac{\partial \alpha_{ij}}{\partial W_K} = 0.
    \end{gather*}
    Additionally, since the heads are uniform the attention bias is zero to begin with
    \begin{equation*}
        \abias^c = \frac{1}{T} \sum_{1 \leq j \leq i \leq T} c_{ij} \alpha_{ij} = \frac{1}{T}\sum_{i = 1}^T \frac{1}{i }\sum_{j \leq i} c_{ij} =  \frac{1}{T}\sum_{i = 1}^T \frac{1}{i} \cdot 0 = 0.
    \end{equation*}
\end{proof}

\begin{remark}
    Note that part (4) of the proposition holds for RoPE transformers as well. Parts (1), (2) and (3) \emph{do not}. The relative rotations break attention uniformity and thus changing the magnitude of $\norm{W_Q}$ and $\norm{W_K}$ can affect the attention weights. This is formally demonstrated in the next section.
\end{remark}

\subsection{Proof of Proposition~\ref{prop:rope_lowerbound}}\label{appendix:rope_lowerbound}

\RoPELowerbound*
\begin{proof}
    Let $x_1 = \cdots = x_T = x \in \sR^d$ be the inputs to a RoPE attention head, and let $W_Q, W_K \in \sR^{d_k \times d}$ be the query and key projection parameters. Since the projection maps are shared across tokens, the queries and keys are constant as well:
    \begin{equation*}
        q_i = W_Q x_i = W_Q x = q, \qquad k_i = W_K x_i = W_K x = k.
    \end{equation*}
    Set the positional bias weights to be
    \begin{equation*}
        c_{ij} = \alpha_{ij} - \tfrac{1}{i}.
    \end{equation*}
    Since $\sum_{j \leq i} \alpha_{ij} = 1$, we have $\sum_{j \leq i} c_{ij} = 0$ as required. The positional bias $\abias^c$ is
    \begin{equation*}
        \abias^c = \frac{1}{T}\sum_{i = 1}^T\sum_{j \leq i} (\alpha_{ij}^2 - \tfrac{1}{i}\alpha_{ij}) =  \frac{1}{T}\sum_{i = 1}^T\Big(\sum_{j \leq i} \alpha_{ij}^2 - \tfrac{1}{i}\Big).
    \end{equation*}
    By Cauchy-Schwarz,
    \begin{equation*}
        1 = \Big(\sum_{j \leq i} \alpha_{ij} \cdot 1\Big)^2 \leq \Big(\sum_{j \leq i} \alpha_{ij}^2\Big)\Big(\sum_{j \leq i} 1\Big) = i \sum_{j \leq i} \alpha_{ij}^2,
    \end{equation*}
    with equality only when $\alpha_{i1} = \cdots = \alpha_{ii}$. Therefore,
    \begin{equation*}
        \sum_{j \leq i} \alpha_{ij}^2 \geq \frac{1}{i},
    \end{equation*}
    with equality iff $\alpha_{ij} = 1/i$ is uniform. Therefore, $\abias^c > 0$ unless $\alpha_{ij}$ is uniform for all $i$. The following lemma asserts that this is not the case
    \begin{restatable}{lemma}{NonUnif}\label{lem:non_unif}
        For any non-degenerate RoPE head and input embeddings $x_1 = \cdots = x_t = x$, there exists $i \geq 1$ such that $s_{i1}, \dots, s_{ii}$ and $\alpha_{i1}, \dots, \alpha_{ii}$ are not uniform.
    \end{restatable}
    The proof of Lemma~\ref{lem:non_unif} is at the end of this subsection. As for $\nabla_\params \abias^c$, rewrite $\abias^c$ as
    \begin{equation*}
        \abias^c=\frac{1}{T} \sum_{i=1}^T \Big(\sum_{j\leq i}\alpha_{ij}^2-\tfrac{1}{i}\Big) = \frac{1}{T} \sum_{i=1}^T F_i -\frac{1}{T} \sum_{i=1}^T \tfrac{1}{i},
    \end{equation*}
    so the dependence in the parameters $\params$ is entirely through 
    \begin{equation*}
        F_i := \sum_{j \leq i}\alpha_{ij}^2.
    \end{equation*}
    From the definition of RoPE, we have
    \begin{equation*}
        \alpha_{ij} = \softmax(s_{i1}, \dots, s_{ii})_j, \qquad s_{ij} = 
        \frac{1}{\sqrt{d_k}}q^\top R^{j-i}k.
    \end{equation*}
    Consider scaling $q$ by a scalar $\lambda>0$: $q \mapsto \lambda q$. For fixed prefix $i$, define
    \begin{equation*}
        Z_i(\lambda) := \sum_{j \leq i}e^{\lambda s_{ij}},\qquad
        \alpha_{ij}(\lambda)=\frac{e^{\lambda s_{ij}}}{Z_i(\lambda)},\qquad
        F_i(\lambda):=\sum_{j\le i}\alpha_{ij}(\lambda)^2.
    \end{equation*}
    Then
    \begin{equation*}
        F_i(\lambda) = \frac{Z_i(2\lambda)}{Z_i(\lambda)^2} \quad \Longrightarrow \quad \frac{d}{d\lambda} \log F_i(\lambda)=2 \underbrace{\E_{j \sim \alpha_i(2\lambda)}[s_{ij}]}_{A_i'(2\lambda)}
        -2 \underbrace{\E_{j \sim \alpha_i(\lambda)}[s_{ij}]}_{A_i'(\lambda)},
    \end{equation*}
    where $A_i(\lambda) := \log Z_i(\lambda)$ is the log-partition function. The second derivative of the log partition function is the logit variance
    \begin{equation*}
        A_i''(\lambda) = \frac{Z_i''(\lambda)Z_i(\lambda) - (Z_i'(\lambda))^2}{(Z_i(\lambda))^2} = \frac{\sum_{j \leq i} s_{ij}^2 e^{\lambda s_{ij}}}{Z_i(\lambda)} - \left(\frac{\sum_{i \leq j} s_{ij} e^{\lambda s_{ij}}}{Z_i(\lambda)}\right)^2 = \Var_{j \sim \alpha_i(\lambda)}(s_{ij})
    \end{equation*}
    therefore $A_i''(\lambda)=\mathrm{Var}_{\alpha_i(\lambda)}(s_{i\cdot})>0$ since from Lemma~\ref{lem:non_unif} $s_{ij}$ are not all equal and $\alpha_{ij}(\lambda) > 0$. Thus, $A_i'(\lambda)$ is strictly increasing in $\lambda$. Hence, for any $i$ with \emph{non-constant} logits,
    \begin{equation*}
        \frac{d}{d\lambda}F_i(\lambda) = F_i(\lambda) \cdot 2\big(A_i'(2\lambda)-A_i'(\lambda)\big) >0, 
    \end{equation*}
    and in particular at $\lambda = 1$,
    \begin{equation*}
        \frac{d}{d\lambda}F_i(\lambda)\Big|_{\lambda=1} > 0 .
    \end{equation*}

    By the chain rule for $q \mapsto \lambda q$,
    \begin{equation*}
    \frac{d}{d\lambda}F_i(\lambda)\Big|_{\lambda=1}
    = \nabla_q F_i(q)^\top \cdot  q .
    \end{equation*}
    Thus $\nabla_q F_i(q)\neq 0$ (otherwise the dot product with $q$ couldn’t be strictly positive). Finally, since $q = W_Qx$,
    \begin{equation*}
        \nabla_{W_Q}F_i = \nabla_q F_i x^\top,
    \end{equation*}
    and with $x\neq 0$ we get $\norm{\nabla_\params F_i} \geq \norm{\nabla_{W_Q}} F_i > 0$. Therefore
    \begin{equation*}
        \nabla_\params \abias^ c =\frac{1}{T}\sum_{i=1}^T \nabla_\params F_i
    \end{equation*}
    has strictly positive norm (a sum of nonzero matrices sharing the same nonzero right factor $x^\top$ cannot be the zero matrix unless all left factors vanish, which they don’t for $i \ge 2$).
\end{proof}

To conclude this section, we now prove Lemma~\ref{lem:non_unif}.
\NonUnif*
\begin{proof}
    RoPE acts as independent $2\times 2$ rotations on disjoint coordinate pairs. Thus
    \begin{equation*}
        R^{\Delta} = \bigoplus_{m=1}^{M} R\big(\Delta \omega_m\big), \quad M=d_k/2
    \end{equation*}
    with pairwise distinct frequencies $\omega_m \in(0,2\pi)$. Decompose
    \begin{equation*}
        q=(q_1,\dots,q_M),\quad k=(k_1,\dots,k_M),\qquad a_m,b_m \in \sR^2,
    \end{equation*}
    so $s_{ij} = f(j-i)$ where
    \begin{equation*}
        f(\Delta) = \frac{1}{\sqrt{d_k}}\sum_{m=1}^{M} q_m^\top R\big(\Delta \cdot \omega_m\big) k_m.
    \end{equation*}
    Let
    \begin{equation*}
        R(\phi)=\begin{pmatrix}\cos\phi&-\sin\phi\\ \sin\phi&\cos\phi\end{pmatrix}, \qquad J=\begin{pmatrix}0&-1\\ 1&0\end{pmatrix}.
    \end{equation*}
    For any $u, v \in \sR^2$,
    \begin{equation*}
        u^\top R(\phi)\,v = (u^\top v)\cos\phi + (u^\top \cdot Jv)\sin\phi.
    \end{equation*}
    Define $A_m := q_m^\top k_m$ and $B_m := q_m^\top J b_m$. Then
    \begin{equation*}
        \begin{split}
            f(\Delta) = \frac{1}{\sqrt{d_k}}\sum_{m=1}^{M}\Big(A_m\cos(\Delta \omega_m) + B_m\sin(\Delta \omega_m)\Big) &= \Re\left(\sum_{m=1}^{M} c_m e^{i\Delta \omega_m}\right) \\
            &= \frac{1}{2}\sum_{m = 1}^M C_m e^{i\Delta \omega_m} + \bar{C}_m e^{-i \Delta \omega_m},
        \end{split}
    \end{equation*}
    where
    \begin{equation*}
        C_m := \frac{A_m-iB_m}{\sqrt{d_k}}.
    \end{equation*}

    Assume $f(\Delta)$ is constant in $\Delta$ for $\Delta = 0, \dots, 2M = d_k$, and denote the constant value by $-\frac{1}{2}C_0$. Then we have
    \begin{equation*}
        \sum_{m = -M}^M C_m e^{i\Delta \omega_m} \equiv 0
    \end{equation*}
    were $C_{-m} := \bar{C}_m$, and $\omega_{-m} = -\omega_m$. Since $\{e^{-i \omega_M}, \dots, e^{-i \omega_1}, 1, e^{i \omega_1}, \dots, e^{i \omega_M}\}$ are all distinct, by Vandermonde's identity this means $C_m = \bar{C}_m = 0$ for $m = 1, \dots, M$, $\Rightarrow A_m = B_m = 0$ for $m = 1, \dots, M$.
    Now $A_m = B_m = 0$ means
    \begin{equation*}
        q_m \perp k_m \quad\text{and} \quad q_m \perp J k_m.
    \end{equation*}
    If $k_m \neq 0$, then $\{k_m, J k_m\}$ spans $\sR^2$, forcing $q_m=0$. Thus for every block $m$, either $q_m = 0$ or $k_m = 0$, which results in a degenerate RoPE head, contradicting the assumption. Therefore, for $i \geq d_k + 1$ the attention logits $s_{ij}$ are not constant, and thus the attention weight $\alpha_{ij}$ are not constant.
\end{proof}

\subsection{Proof of Theorem~\ref{thm:propagation}}
In this section, we prove Theorem~\ref{thm:propagation}. To do so, we first need to prove a sequence of Propositions and Lemmas. First, we restate the theorem here. 
\Propagation*

Since all weight matrices are drawn from a Gaussian distribution with a fixed variance, there exists a constant $B$, depending only on the architecture, such that with high probability the operator norms of $W_Q$, $W_K$, $W_V$, and $W_O$, as well as the Lipschitz constants of the MLPs and normalization layers are all bounded by $B$. To see this use, e.g. Theorem 4.4.5 from~\citet{vershynin2018high} and the fact that for a two layer MLP $f$, it's Lipschitz constnat is bounded by $\mathrm{Lip}(f) \leq \norm{W_1} \norm{W_2} \mathrm{Lip}(\sigma)$. Let $L$ be the number of layers, and $H$ be the number of attention heads per layer. For any vector sequence $a_i \in \sR^d$ we denote by $\bar{a}_i = \frac{1}{i} \sum_{j \leq i} a_j$ the prefix sum of $a_i$. For real sequences with two indices $a_{ij} \in \sR$ we denote $a_i = (a_{i1}, \dots, a_{ii}) \in \sR^i$ and $\bar{a}_i = \frac{1}{i}\sum_{j \leq i} a_{ij}$.

\begin{proposition}\label{prop:rho}
    Fix a row $i$ in an attention head at the $l$-th layer.
    \begin{equation*}
        \max_{j\le i}\Big|s_{ij}-\bar{s}_i \Big| \leq B^2 \sqrt{H} \Delta_h^{(l)}.
    \end{equation*}
\end{proposition}

\begin{proof}
    Notice that
    \begin{equation*}
        s_{ij}-\bar{s}_i = \frac{1}{\sqrt{d_k}} q_i^\top k_j + \frac{1}{i}\sum_{r \leq i}\frac{1}{\sqrt{d_k}}q_i^\top k_r = \frac{1}{\sqrt{d_k}}q_i^\top (k_j-\bar{k}_i).
    \end{equation*}
    Therefore, by Cauchy-Swartz
    \begin{equation*}
        \Big|s_{ij}-\bar{s}_i \Big| \leq \frac{1}{\sqrt{d_k}} \norm{q_i} \norm{k_j - \bar{k}_i}.
    \end{equation*}
    By the linearity of $W_K$ we get $\norm{k_j-\bar{k}_i}=\norm{W_K(h_j-\bar{h}_i)} \le \norm{W_K}\norm{h_j-\bar{h}_i}\leq \norm{W_K}\Delta_h^{(l)} \leq B \Delta_h^{(l)}$. As for $\norm{q_i} = \norm{W_Q h_i}$, recall that $h_i$ are the output of a normalization layer, and therefore (at initialization) $\norm{h_i} = \sqrt{d}$. Thus, $\norm{q_i} \leq B \sqrt{d}$. Putting it all together gives
    \begin{equation*}
        \Big|s_{ij}-\bar{s}_i \Big| \leq B^2 \sqrt{\frac{d}{d_k}} \Delta_h^{(l)} =
        B^2 \sqrt{H} \Delta_h^{(l)}.
    \end{equation*}
    To finish the proof, take a maximum over $j \leq i$.
\end{proof}
To bound the effect on the attention \emph{probabilities}, we need the following Lemma.
\begin{lemma}\label{lem:softmax}
    For any $b \in \sR^n$,
    \begin{equation*}
        \norm{\softmax(a+b)-\softmax(a)}_1 \leq \norm{b}_\infty.
    \end{equation*}
\end{lemma}
\begin{proof}
    A $C^2$ convex function $f: \sR^n \to \sR$ satisfies $\norm{\nabla f(x)-\nabla f(y)}_1 \leq \norm{x-y}_\infty$ (1-smoothness) if $d^\top \nabla^2 f(x) d \leq \norm{d}_\infty^2$ for all $x, d \in \sR^n$ (see Theorem 2.1.6 in~\citet{optimization}). Take $f(x) = \log\Big(\sum_{i = 1}^n e^{x_i}\Big)$. $f$ is $C^2$, convex and $\nabla f(x) = \softmax(x)$. Therefore, all we need to show is that for all $x, d \in \sR^n$
    \begin{equation*}
        d^\top \nabla \softmax(x) d = d \nabla^2 f(x)d \leq \norm{d}^2.
    \end{equation*}
    and indeed, 
    \begin{equation*}
        \begin{split}
            d^\top \nabla \softmax(x) d &= d^\top\mathrm{diag}(\softmax(x)) d - (\softmax(x)^\top d)^2 \\
            &\leq d^\top \mathrm{diag}(\softmax(x))d \\
            &\leq \norm{d}_\infty^2 \norm{\softmax(x)}_1 \\
            &=\norm{d}_\infty^2,
        \end{split}
    \end{equation*}
    as required.
\end{proof}
Using Lemma~\ref{lem:softmax}, we can bound the uniformity of $\alpha_{ij}$ and the prefix spread of the head outputs.
\begin{proposition}\label{prop:secondorder}
    Let $u_i = \frac1i\1 \in \sR^i$. In any layer $l$,
    \begin{equation}\label{eq:unif_attn}
        \norm{\alpha_i-u_i}_1 \leq B^2 \sqrt{H} \Delta_h^{(l)},
    \end{equation}
    and,
    \begin{equation}\label{eq:unif_outputs}
        \norm{z_i-\bar{v}_i} \le B^3 \sqrt{H} \big(\Delta_h^{(l)}\big)^2.
    \end{equation}
\end{proposition}
\begin{proof}
    To get Equation~\ref{eq:unif_attn}, let $a$ be the constant vector $(\bar{s}_i, \dots, \bar{s}_i) \in \sR^i$ and let $b = s_i - a$. By Lemma~\ref{lem:softmax}
    \begin{equation*}
        \norm{\alpha_i - u_i}_1 = \norm{\softmax(a + b) - \softmax(a)}_1 \leq \norm{b}_\infty.
    \end{equation*}
    Now, notice that $\norm{b}_\infty = \max_{j \leq i} \big|s_{ij} - \bar{s}_i\big|$, therefore Proposition~\ref{prop:rho} gives us the desired inequality. For Equation~\ref{eq:unif_outputs} notice that,
    \begin{equation*}
        z_i-\bar{v}_i=\sum_{j\le i}(\alpha_{ij}-\tfrac1i)\,(v_j-\bar v_i),
    \end{equation*}
    hence
    \begin{equation*}
        \norm{z_i-\bar{v}_i} \leq \max_{j\le i}\norm{v_j-\bar{v}_i} \norm{\alpha_i-u_i}_1 \leq B \Delta_h^{(l)} \cdot B^2 \sqrt{H} \Delta_h^{(l)} = B^3 \sqrt{H} \big(\Delta_h^{(l)}\big)^2.
    \end{equation*}
\end{proof}

We now bound the next layer's spread in terms of the current one. Denote by $\Delta_z^{(l)}:=\max_{i}\max_{j\le i}\norm{z_j-\bar{z}_i}$ the prefix spread of an attention head's output. First, we'll give a bound for $\Delta_z^{(l)}$, and then use this bound to prove the entire propagation result. Before, we need a short lemma.

\begin{lemma}\label{lem:avgstable}
    For any sequence $(x_j)$ and $j \leq i$,
    \begin{equation*}
        \norm{\bar{x}_j-\bar{x}_i} \leq \max_{r \le j} \norm{x_r-\bar{x}_i} \leq \max_{r \le i} \norm{x_r-\bar{x}_i}.
    \end{equation*}
\end{lemma}

\begin{proof}
    $\bar{x}_j-\bar{x}_i= \frac1j \sum_{r\le j}(x_r-\bar{x}_i)$ and triangle inequality.
\end{proof}

\begin{proposition}\label{prop:attn_spread}
    For any layer $1 \leq l \leq L$,
    \begin{equation*}
        \Delta_z^{(l)} \leq 2 B \Delta_h^{(l)} + 2 B^3 \sqrt{H} \big(\Delta_h^{(l)}\big)^2.
    \end{equation*}
\end{proposition}

\begin{proof}
    Fix $i$ and $j\le i$. Write $z_j-\bar{z}_i=(\bar{v}_j-\bar{v}_i) + (z_j-\bar{v}_j) - (\bar{z}_i-\bar{v}_i)$,
    so
    \begin{equation*}
        \norm{z_j-\bar{z}_i} \leq \underbrace{\norm{\bar{v}_j - \bar{v}_i}}_{=\mathrm{(a)}} + \underbrace{\norm{z_j-\bar v_j}}_{=\mathrm{(b)}} + \underbrace{\norm{\bar{z}_i-\bar{v}_i}}_{=\mathrm{(c)}}.
    \end{equation*}
    By Lemma~\ref{lem:avgstable}, 
    \begin{equation*}
        \mathrm{(a)} = \norm{\bar{v}_j - \bar{v}_i} \leq \max_{r \leq i}\norm{v_r-\bar{v}_i} \leq \norm{W_V} \Delta_h^{(l)} \leq B \Delta_h^{(l)}.
    \end{equation*}
    By Proposition~\ref{prop:secondorder}
    \begin{equation*}
        \mathrm{(b)} \leq B^3 \sqrt{H} \big(\Delta_h^{(l)}\big)^2.
    \end{equation*}
    As for (c), Notice that,
    \begin{equation*}
       \bar{z}_i - \bar{v}_i =  \frac{1}{i}\sum_{r \leq i} (z_r - \bar{v}_i) = \frac{1}{i}\sum_{r \leq i} \left((z_r - \bar{v}_r) + (\bar{v}_r - \bar{v}_i)\right), 
    \end{equation*}
    therefore by the triangle inequality, Proposition~\ref{prop:secondorder}, and Lemma~\ref{lem:avgstable},
    \begin{equation*}
        \begin{split}
            \mathrm{(c)} \leq \frac{1}{i}\sum_{r \leq i} \norm{z_r - \bar{v}_r} + \frac{1}{i} \sum_{r \leq i} \norm{\bar{v}_r - \bar{v}_i} &\leq B^3\sqrt{H} \big(\Delta_h^{(l)}\big)^2 + \frac{1}{i}\sum_{r \leq i} \max_{k \leq i}\norm{v_k - \bar{v}_i} \\
            &= B^3\sqrt{H} \big(\Delta_h^{(l)}\big)^2 + \max_{k \leq i} \norm{v_k - \bar{v}_i} \\
            &\leq B^3\sqrt{H}\big(\Delta_h^{(l)}\big)^2 + B \Delta_h^{(l)}
        \end{split}
    \end{equation*}
    
    To finish the proof, take the maximum over $i$ and $j \leq i$.
\end{proof}

\begin{proposition}[Full Transformer block recursion]\label{prop:layer_recur}
    There exist constants $A_1, A_2$ depending only on $B$, and $H$, such that
    \begin{equation*}
        \Delta_h^{(l+1)} \leq A_1 \Delta_h^{(l)} + A_2\big(\Delta_h^{(l)}\big)^2.
    \end{equation*}
\end{proposition}
\begin{proof}
    From Proposition~\ref{prop:attn_spread}, the single-head spread is bounded by a linear term $2 B\Delta_h$ plus a quadratic term $2B^3\sqrt{H}$. Concatenation and $W_O$ multiply by at most $\|W_O\|$ (up to a fixed constant depending on number of heads). Adding the residual preserves a linear contribution in $\Delta_h^{(\ell)}$. The positionwise LayerNorm/MLP, being $B$-Lipschitz, scales the spread by at most $B$. Collecting the constants into $A_1$ and, $A_2$ gives the desired result.
\end{proof}
We can now proof the full propagation result.
\begin{theorem}\label{thm:propagation_formal}
    For any finite depth $L$, there exists $\varepsilon>0$ (depending on $B$, $L$, and $H$) such that if $\Delta_h^{(1)} \leq \varepsilon$, then for all $l \leq L$,
    \begin{equation*}
        \Delta_h^{(l)} \leq C \Delta_h^{(1)} \leq C \varepsilon,
    \end{equation*}
    with $C=C(B, L, H)$.
\end{theorem}
\begin{proof}
    By Proposition~\ref{prop:layer_recur}, $\Delta_h^{(l+1)} \leq A_1 \Delta_h^{(l)} + A_2 (\Delta_h^{(l)})^2$. Choose $\varepsilon \leq \min\{1, (A_1/A_2)\}$ so that $A_2 \Delta_h^{(l)} \le A_1$. Then $\Delta_h^{(l+1)} \leq 2A_1 \Delta_h^{(l)}$. Induction yields $\Delta_h^{(l)} \leq (2A_1)^{l-1}\Delta_h^{(1)} \leq C \Delta_h^{(1)}$ for $l \leq L$ with $C =(2A_1)^{L-1}$.
\end{proof}
This conclude the first part of the proof, regarding uniformity propagation across depth. Note that the bounds in the proof \emph{do not} depend on the number of tokens in the input sequence. 

\paragraph{$\abias^c$ bound.}
Recall that,
\begin{equation*}
    \abias^c = \frac{1}{T}\sum_{i = 1}^T\sum_{j \leq i} \alpha_{ij} c_{ij}
\end{equation*}
where $c_{ij}$ are centered positional weights, i.e. $\sum_{j \leq i} c_{ij} = 0$. For any such $c_{ij}$ we have
\begin{equation*}
    \begin{split}
        \Big|\abias^c\Big| &= \frac{1}{T} \Big|\sum_{i = 1}^T\sum_{j \leq i} \alpha_{ij} c_{ij}\Big| = \frac{1}{T} \Big|\sum_{i = 1}^T\sum_{j \leq i} \alpha_{ij} c_{ij}\Big| = \frac{1}{T} \Big|\sum_{i = 1}^T\sum_{j \leq i} (\alpha_{ij}-\tfrac{1}{i}) c_{ij} + \sum_{i = 1}^T\underbrace{\sum_{j \leq i} \tfrac{1}{i} c_{ij}}_{= 0}\Big| \\
        &\leq \frac{1}{T}\sum_{i = 1}^T \sum_{j \leq i} |\alpha_{ij} - \tfrac{1}{i}||c_{ij}| \leq \Big(\underbrace{\max_{1 \leq j \leq i \leq T}|c_{ij}|}_{C}\Big) \frac{1}{T}\sum_{i = 1}^T \norm{\alpha_i - u_i}_1 \\
        &\leq C B^2 \sqrt{H} \Delta_h^{(l)} = \gO(\varepsilon).
    \end{split}
\end{equation*}

\paragraph{Q/K gradient bounds.}
Let $g_{ij} = \partial \abias^c / \partial s_{ij}$. We have
\begin{equation*}
    g_{ij} = \tfrac{1}{T}\alpha_{ij}(c_{ij} - c_i^\alpha),
\end{equation*}
where $c_i^\alpha = \sum_{p \leq i} \alpha_{ip} c_{ip}$.
\begin{lemma}\label{lem:rowsum}
    For every $i$, $\sum_{j\le i} g_{ij}=0$, and therefor for any vectors $a_j$
    \begin{equation*}
        \sum_{j \leq i}g_{ij}a_j = \sum_{j \leq i}g_{ij}(a_j - \bar{a}_i).
    \end{equation*}
\end{lemma}
\begin{proof}    
    First notice that
    \begin{equation*}
        \sum_{j \leq i}g_{ij} = \frac{1}{T}\sum_{j \leq i}\alpha_{ij} (c_{ij} - c_{ij}^\alpha) = \frac{1}{T}\E_{j \sim \alpha_i}[c_{ij} - \E_{p \sim \alpha_i}[c_{ip}]] = 0.
    \end{equation*}
    For the second part, observe that
    \begin{equation*}
        \sum_{j \leq i} g_{ij}(a_j - \bar{a}_i) =\sum_{j \leq i} g_{ij}a_j - \bar{a}_i\sum_{j \leq i}g_{ij} = \sum_{j \leq i} g_{ij}a_j.
    \end{equation*}
\end{proof}
Now, from direct computation and an application of Lemma~\ref{lem:rowsum}, we have
\begin{align*}
    \frac{\partial \abias^c}{\partial W_Q} &= \frac{1}{\sqrt{d_k}}\sum_{i=1}^T \Big(\sum_{j \le i}g_{ij} k_j\Big)h_i^\top = \frac{1}{\sqrt{d_k}}\sum_{i=1}^T \sum_{j \le i}g_{ij} (k_j - \bar{k}_i)h_i^\top,
    \\
    \frac{\partial \abias^c}{\partial W_K} &= \frac{1}{\sqrt{d_k}} \sum_{i=1}^T q_i \Big(\sum_{j \leq i} g_{ij} h_j^\top \Big) = \frac{1}{\sqrt{d_k}} \sum_{i = 1}^T\sum_{j \leq i} g_{ij}q_i(h_j - \bar{h}_i)^\top.
\end{align*}
Let's analyse the norm:
\begin{equation*}
    \begin{split}
        \Bnorm{\frac{\partial \abias}{\partial W_K}} &= \Bnorm{\frac{1}{\sqrt{d_k}} \sum_{i = 1}^T\sum_{j \leq i} g_{ij}q_i(h_j - \bar{h}_i)^\top} \leq \frac{1}{\sqrt{d_k}} \sum_{i = 1}^T\sum_{j \leq i} |g_{ij}| \norm{q_i} \norm{h_j - \bar{h}_i} \\
        &\leq B\sqrt{H} \Delta_h^{(l)}\sum_{i = 1}^T\sum_{j \leq i} |g_{ij}| \\
        &\leq \frac{B\sqrt{H} \Delta_h^{(l)}}{T} \Big(\sum_{i = 1}^T\sum_{j \leq i} |(\alpha_{ij} - \tfrac{1}{i})(c_{ij} - c_i^\alpha)| + \sum_{i = 1}^T\sum_{j \leq i}\tfrac{1}{i}|c_{ij} - c_i^\alpha|\Big) \\
        &\leq B\sqrt{H} \Delta_h^{(l)}\Big(B^2 \sqrt{H} \Delta_h^{(l)}C + C\Big) = \gO(\varepsilon), 
    \end{split}
\end{equation*}
where $C = \max_{1 \leq j \leq i \leq T} |c_{ij} - c_i^\alpha| \leq \max_{1 \leq j \leq i \leq T} |c_{ij}|$. An analogous result holds for $W_Q$, 
\begin{equation*}
    \begin{split}
        \Bnorm{\frac{\partial \abias}{\partial W_Q}} &= \Bnorm{\frac{1}{\sqrt{d_k}}\sum_{i=1}^T \sum_{j \le i}g_{ij} (k_j - \bar{k}_i)h_i^\top} \leq \frac{1}{\sqrt{d_k}}\sum_{i = 1}^T\sum_{j \leq i}|g_{ij}| \norm{k_j - \bar{k}_i} \norm{h_i} \\
        &\leq B\sqrt{H} \Delta_h^{(l)} \sum_{1 \leq j \leq i \leq T}|g_{ij}| \\
        &\leq \frac{B\sqrt{H} \Delta_h^{(l)}}{T} \Big(\sum_{i = 1}^T\sum_{j \leq i} |(\alpha_{ij} - \tfrac{1}{i})(c_{ij} - c_i^\alpha)| + \sum_{i = 1}^T\sum_{j \leq i}\tfrac{1}{i}|c_{ij} - c_i^\alpha|\Big) \\
        &\leq B\sqrt{H} \Delta_h^{(l)}\Big(B^2 \sqrt{H} \Delta_h^{(l)}C + C\Big) = \gO(\varepsilon). 
    \end{split}
\end{equation*}
This concludes the proof of Theorem~\ref{thm:propagation}.

\section{Experimental details} \label{appendix: experimental details}

\subsection{Training}
\paragraph{\ouralgo from a RoPE transformer trained from scratch.}
For the first part of our experimental evaluation, we train a small RoPE transformer with almost half a billion parameters on FineWeb~\citep{fineweb} for over 16B tokens with a sequence length of 1024. We note this is well over 2 times the chinchilla optimal number of tokens from \citet{chinchilla}. We use a \textsc{Qwen2}~\citep{qwen2} tokenizer and follow the specifications (number of layers/hidden dimensions) from the 0.5B model from the same family. We implemented all our baselines on top of this architecture, pretraining them for the same large number of tokens. We use the AdamW optimizer~\cite{adamw} with a small warmup phase of 520 steps, a batch size of 1024, a peak learning rate of $3.0\times10^{-4}$, and a cosine decay thereafter. For \ouralgo we followed a similar optimization setup, but only training for 2B total tokens using a shorter warmup of 70 steps and a slightly larger learning rate of $1.0\times10^{-3}$ to compensate for the shorter training budget. We provide a full list of hyperparameters and training specifications for this setting in the left column of Table~\ref{tab:full_hparams}.

\paragraph{\ouralgo from a pretrained \textsc{SmolLM} .}
For the second part of our experimental evaluation, we use a \textsc{SmolLM}~\citep{smollm} with around 362 million parameters already extensively pretrained on the SmolLM corpus~\citep{smollm_corpus} for over 600B tokens with a sequence length of 2048 -- almost 100 times the chinchilla optimal number. This model used a \textsc{GPT2}~\citep{gpt2} tokenizer and its architecture was designed to be similar to models of the \textsc{Llama2} family~\citep{llama2}. While not all training details have been disclosed,  \citet{smollm} explicitly mentions using the AdamW optimizer~\cite{adamw}, a batch size of 512, a peak learning rate of $3.0\times10^{-3}$, and a cosine decay thereafter. 
For \ouralgo we again tried to follow a similar optimization setup, across our different 30B/60B/120B training regimes, introducing a short warmup of 490 steps and a slightly lower learning rate of $1.0\times10^{-3}$ as we found their reported $3.0\times10^{-3}$ led to instabilities from the small batch size. Given the more extended training period, we used a simple QKNorm~\citep{qknorm} after dropping the positional embeddings, which we found beneficial to mitigate sporadic instabilities from large gradients. We note that preliminary experiments showed that normalizing only the queries led to even faster learning and also successfully stabilized long training. We believe further exploration of this new Q-norm method could be an exciting direction for future work to train transformers without positional embeddings at even larger scales. We provide a full list of hyperparameters and training specifications for this setting in the right column of Table~\ref{tab:full_hparams}.

\begin{table}[H]
\caption{Architectures, optimization, and other training setup hyperparameters for pretraining our RoPE transformer, \textsc{SmolLM}, and our two new \ouralgo phases.} \label{tab:full_hparams}
\centering
\small
\begin{tabular}{lcc}
\toprule
\textbf{Pretraining and DroPE Hyperparameter} &
\text{RoPE transformer} &
\textsc{SmolLM}\\
\midrule
\multicolumn{3}{l}{\textbf{Model architectures}} \\
\midrule
Model parameters & 494M & 362M \\
Model parameters w/o embeddings & 358M & 315M \\
Hidden size & 896 & 960 \\
Hidden MLP size & 4864 & 2560 \\
Hidden activation & SiLU & SiLU \\
Number of hidden layers & 24 & 32 \\
Number of attention heads & 14 & 15 \\
Number of key--value heads & 2 & 5 \\
Head dimension & 64 & 64 \\
Attention bias & false & false \\
Attention dropout & 0.0 & 0.0 \\
Initializer range & 0.02 & 0.02 \\
RoPE $\theta$ & 1{,}000{,}000 & 10{,}000 \\
Tied word embeddings & true & true \\
Output router logits & true & true \\
Computation dtype & bfloat16 & bfloat16 \\
Tokenizer &\textsc{Qwen2} & \textsc{GPT2} \\
\midrule
\multicolumn{3}{l}{\textbf{Pretraining setup}} \\
\midrule
Optimizer & AdamW & AdamW \\
Learning rate & $3.0\times10^{-4}$ & $3\times10^{-3}$ \\
Weight decay & 0.1 & 0.1 \\
Adam parameters $(\beta_1,\beta_2,\epsilon)$ & (0.9, 0.95, $1\!\times\!10^{-8}$) & (0.9, 0.95, $1\!\times\!10^{-8}$) \\
Learning rate scheduler & Cosine decay & Cosine decay \\
Warmup steps & 520 & N/A \\
\midrule
Maximum sequence length & 1024 & 2048 \\
Global train batch size (sequences) & 1024 & 512 \\
Tokens per training step & 1{,}048{,}576 & 1{,}048{,}576 \\
Total tokens & 16.8B & 600B \\
Dataset & fineweb & smollm-corpus \\
\midrule
\multicolumn{3}{l}{\textbf{\ouralgo setup}} \\
\midrule
QK-norm & False & True\\
Optimizer & AdamW & AdamW \\
Learning rate & $1.0\times10^{-3}$ & $1.0\times10^{-3}$ \\
Weight decay & 0.1 & 0.1 \\
Adam parameters $(\beta_1,\beta_2,\epsilon)$ & (0.9, 0.95, $1\!\times\!10^{-8}$) & (0.9, 0.95, $1\!\times\!10^{-8}$) \\
Learning rate scheduler & Cosine decay & Cosine decay \\
Warmup steps & 70 & 490 \\
\midrule
Maximum sequence length & 1024 & 2048 \\
Global train batch size (sequences) & 1024 & 512 \\
Tokens per training step & 1{,}048{,}576 & 1{,}048{,}576 \\
Total tokens & 2.10B & 31.46B/62.9B/125.8B \\
Dataset & fineweb & fineweb-edu \\
\bottomrule
\end{tabular}
\end{table}

\subsection{Evaluation}
\paragraph{Needle-in-a-haystack.}
We evaluate long-context retrieval using the \emph{needle-in-a-haystack} (NIAH) setup, which places a short ``needle'' inside a long distractor “haystack.” Following prior work~\citep{niah_repo}, our haystack is a random excerpt from Paul Graham’s essays, and each needle is a seven-digit ``magic number'' paired with a short key/descriptor. We study three variants:
\begin{itemize}
    \item \textbf{(Standard NIAH)} We insert a single needle and prompt the model to retrieve it.
    \item \textbf{Multi-Query NIAH:} We insert multiple (key, value) pairs and prompt the model to return as many values as possible for a given list of keys. For example: \texttt{The special magic numbers for whispering-workhorse and elite-butterfly mentioned in the provided text are:}.
    \item \textbf{(Multi-Key NIAH)} We insert multiple (key, value) pairs but query for a single key, e.g., \texttt{The special magic number for elite-butterfly mentioned in the provided text is:}
    \item \textbf{(Multi-Value NIAH)} We associate multiple values with one key and ask for all of them without pointing to specific positions, e.g., \texttt{What are all the special magic numbers for cloistered-colonization mentioned in the provided text?}
\end{itemize}
Inserted needles and example targets are formatted in natural language, for instance, two examples include \texttt{One of the special magic numbers for whispering-workhorse is: 1019173} and \texttt{One of the special magic numbers for elite-butterfly is: 4132801}. For the standard NIAH variant, we report the average success rate over all possible needle depths. For the multiple needles NIAH variants, we always insert four (key, value) needle pairs, placed at random sequence locations. Unless otherwise noted, we use greedy decoding (logit temperature~$=0$) for reproducibility.

\paragraph{Long-context evaluations.}
We use standard implementations of PI, RoPE-NTK, and YaRN. For tasks that require a fixed maximum context length (e.g., NIAH at $2\times$ the training context), we set the \emph{extension factor} $s$ manually. For settings that require reasoning across multiple context lengths and extended generations, we employ a \emph{dynamic scaling} schedule that adjusts $\gamma$ as a function of the generation length as detailed in~\citet{yarn}.

For \ouralgo, we follow \citet{wang2024length} and apply softmax \emph{temperature scaling} when evaluating on longer sequences. In practice, we tune a single scalar logit scale (equivalently, the inverse temperature) on a held-out set at the target length. Analogous to \citep{yarn}, we fit this coefficient by minimizing perplexity to obtain the optimal scaling. For the \ouralgo model trained from scratch, the best-performing scale is
\begin{equation*}
    \beta^\star = 1 + 0.412 \ln(s),
\end{equation*}
and for \textsc{SmolLM-DroPE} the optimal scale is
\begin{equation*}
    \beta^\star = 1 + 0.103 \ln(s),
\end{equation*}
Where $s = C_\mathrm{test}/C_\mathrm{train}$ is the context extension factor. Unless otherwise specified, all other decoding settings are held fixed across lengths.

\paragraph{Language modeling benchmarks.}
We evaluate \textsc{SmolLM} and \textsc{SmolLM-\ouralgo} on six standard multiple-choice benchmarks using the \textsc{LightEval} harness~\citep{lighteval}: \textbf{ARC-E/C:} grade-school science QA split into Easy and Challenge sets, the latter defined by questions that defeat simple IR and co-occurrence baselines~\citep{bench_1_arc}; \textbf{HellaSwag:} adversarially filtered commonsense sentence completion that is easy for humans but challenging for LMs~\citep{bench_2_hellaswag}; \textbf{OpenBookQA:} combining a small ``open book'' of science facts with broad commonsense to answer 6K questions~\citep{bench_3_openbook_qa}; \textbf{PIQA:} two-choice physical commonsense reasoning~\citep{bench_4_piqa}; and \textbf{WinoGrande:} a large-scale, adversarial Winograd-style coreference/commonsense benchmark~\citep{bench_5_winogrande}. We follow the harness defaults for prompt formatting, decoding, and scoring, and do not perform any task-specific fine-tuning or data adaptation.

\begin{table}[H]
    \centering
    \small
    \caption{\textbf{\ouralgo matches base model in-context performance.} Comparison of the pretrained \textsc{SmolLM-360M} and \textsc{SmolLM-1.7B} models with \textsc{SmolLM-360M-\ouralgo} and \textsc{SmolLM-1.7B-\ouralgo} respectively. Modes are evaluated on a variety of LM benchmarks across question answering and reasoning tasks.}\label{tab:in_context_results}
    \resizebox{\textwidth}{!}{
        \begin{tabular}{lcccccc|c}
            \toprule
            \textbf{Model} & \textbf{ARC-E} & \textbf{ARC-C} & \textbf{HellaSwag} & \textbf{OpenBookQA} & \textbf{PIQA} & \textbf{Winogrande} & \textbf{Avg.} \\
            \midrule
            \textsc{SmolLM-360M} & $65.6$ & $36.0$  & $53.8$ & $37.2$ & $\mathbf{72.0}$ & $\mathbf{53.7}$ & $53.1$ \\
            \textsc{SmolLM--360M-\ouralgo} &  $\mathbf{67.3}$ & $\mathbf{37.6}$ & $\mathbf{53.9}$ & $\mathbf{38.0}$ & $71.5$ & $52.3$ & $\mathbf{53.4}$\\
            \midrule
            \textsc{SmolLM-1.7B} & $77.50$ & $\mathbf{44.0}$ & $64.10$ & $42.60$ & $\mathbf{77.30}$ & $56.00$ & $60.25$ \\
            \textsc{SmolLM-1.7B-\ouralgo} & $\mathbf{77.70}$ & $42.9$ & $\mathbf{65.90}$ & $\mathbf{43.00}$ & $77.10$ & $\mathbf{57.10}$ & $\mathbf{60.62}$ \\
            \bottomrule
        \end{tabular}
    }
\end{table}
\section{Additional experimental results} \label{appendix: extra experiments}

\subsection{Additional recalibration ablations}\label{apx:rec_ablation}
\textbf{When should we start recalibration?} In this setup, we train a 500M-parameter transformer on 16B tokens and remove its PEs during training. We vary the training step at which recalibration is activated. We consider four recipes:
\begin{itemize}
    \item Dropping PEs from step 0 (\emph{NoPE transformer}),
    \item Dropping PEs at step 8K,
    \item Dropping PEs at step 14K,
    \item Dropping PEs at step 16K (\emph{RoPE transformer}, i.e., no dropping during training).
\end{itemize}

Table~\ref{tab:pretrain_recalib} reports the final validation perplexity for each setting.

\begin{table}[htbp]
\centering
\small
\caption{Validation perplexity for a 500M-parameter transformer trained on 16B tokens, when dropping positional encodings at different stages of pretraining.}
\label{tab:pretrain_recalib}
\begin{tabular}{lcccc}
\toprule
& \textbf{\ouralgo @ 0K (NoPE)} & \textbf{\ouralgo @ 8K} & \textbf{\ouralgo @ 14K} & \textbf{\ouralgo @ 16K (RoPE)} \\
\midrule
Validation perplexity & 23.77 & 22.42 & 21.73 & \textbf{21.72} \\
\bottomrule
\end{tabular}
\end{table}

We find that this ablation further strengthens our theoretical observation that DroPE should be integrated later in training. Our analysis in Section~3 suggests that NoPE transformers struggle to train efficiently, whereas retaining RoPE for most of the training benefits optimization. Consistent with this, we observe that dropping the positional encoding only at the very end of pretraining (DroPE @ 16K) yields the best validation perplexity, while earlier dropping steadily degrades performance.

Finally, we emphasize that in this setup, DroPE does not incur additional training cost: the total number of optimization steps is unchanged, and once the positional encoding is removed, training becomes slightly faster due to skipping the RoPE rotation operations in attention.

\subsection{Performance at different context extension factors}\label{apx:perf_v_exp}
\textbf{Average LongBench scores and tasks breakdowns.}
The following tables provide average results over the entire LongBench benchmark (Table~\ref{tab:longbench_avg}), and provide a performance breakdown per input length for the MultiFieldQA and MuSiQue tasks from LongBench (Tables~\ref{tab:multifieldqa_buckets} and~\ref{tab:musique_buckets}). 

\begin{table}[H]
    \centering
    \small
    \caption{Average performance over all LongBench tasks for different RoPE scaling methods.}
    \label{tab:longbench_avg}
    \begin{tabular}{lc}
        \toprule
        \textbf{Method} & \textbf{Avg. LongBench score} \\
        \midrule
        \textsc{SmolLM}        &  2.59 \\
        \textsc{SmolLM} + PI        &  2.48 \\
        \textsc{SmolLM} + RoPE-NTK  & 12.21 \\
        \textsc{SmolLM} + YaRN      & 13.07 \\
        \cmidrule(lr){1-2}
        \textsc{SmolLM-\ouralgo} & \textbf{13.81} \\
        \bottomrule
    \end{tabular}
\end{table}

\begin{table}[htbp]
    \centering
    \small
    \caption{MultiFieldQA performance across context length buckets for \textsc{SmolLM} variants.}
    \label{tab:multifieldqa_buckets}
    \begin{tabular}{lccc}
        \toprule
        \textbf{Model} &
        \textbf{0--4K (0--2$\times$ ctx)} &
        \textbf{4--8K (2--4$\times$ ctx)} &
        \textbf{8--16K (4--8$\times$ ctx)} \\
        \midrule
        SmolLM-DroPE & 32.82 & 24.73 & 30.07 \\
        SmolLM-NTK   & 34.25 & 22.30 & 21.63 \\
        SmolLM-YaRN  & 33.96 & 22.91 & 20.08 \\
        \bottomrule
    \end{tabular}
\end{table}

\begin{table}[htbp]
    \centering
    \small
    \caption{MuSiQue performance across context length buckets for \textsc{SmolLM} variants.}
    \label{tab:musique_buckets}
    \begin{tabular}{lcccc}
        \toprule
        \textbf{Model} &
        \textbf{0--4K (0--2$\times$ ctx)} &
        \textbf{4--8K (2--4$\times$ ctx)} &
        \textbf{8--16K (4--8$\times$ ctx)} &
        \textbf{16--32K (8--16$\times$ ctx)} \\
        \midrule
        SmolLM-DroPE & 50.00 &  6.11 &  8.05 & 16.67 \\
        SmolLM-NTK   &  0.00 &  4.36 &  3.36 &  0.00 \\
        SmolLM-YaRN  &  0.00 & 19.68 &  3.13 &  7.14 \\
        \bottomrule
    \end{tabular}
\end{table}

\textbf{Needle-in-a-haystack performance at larger extension factors.}
To directly measure the effect of the context extension factor on downstream performance, we use standard needle-in-a-haystack evaluations at $2\times$, $4\times$, and $8\times$ original context length. We use \textsc{SmolLM} as the base model, and additionally compare against LongRoPE2~\citep{longrope2} since it was specifically evaluated on NIAH tasks.

\begin{table}[h]
    \caption{\ouralgo outperforms RoPE-scaling methods on long needle-in-a-haystack tasks.}\label{tab:long_niah}
    \centering
    \small 
    \resizebox{0.95\textwidth}{!}{
        \begin{tabular}{lccc}
            \toprule
            \textbf{Method} & \textbf{$2\times$ original context} & \textbf{$4\times$ original context} & \textbf{$8\times$ original context} \\
            \midrule
            \textsc{SmolLM} + RoPE-NTK & $29.84$ & $14.37$ & $7.19$ \\
            \textsc{SmolLM} + YaRN  & $48.25$ & $25.62$ & $12.18$ \\
            \textsc{SmolLM} + LongRoPE2  & $44.20$ & $26.20$ & $16.45$ \\
            \midrule
            \textsc{SmolLM-DroPE} & $\mathbf{74.92}$ & $\mathbf{55.00}$ & $\mathbf{52.20}$ \\
            \bottomrule
        \end{tabular}
    }
    \vspace{-10pt}
\end{table}

\subsection{The effect of QKNorm}\label{apx:qknorm}
We introduce QKNorm in the recalibration phase as an optimization-stability mechanism to enable training with higher learning rates, following recent practices in large-scale model training such as OLMo2~\citep{qknorm_olmo2} and Qwen3~\citep{qwen3}, where normalization is used to stabilize gradients and mitigate loss spikes.

To assess the interaction between QK Norm and DroPE, we conducted a controlled ablation study on the SmolLM-360M model using six configurations: three learning rates ($3\times10^{-5}$, $3\times10^{-4}$, $10^{-3}$), each trained with and without QK Norm. The results, summarized in Table~\ref{tab:qk-ablation}, yield two main observations:
\begin{itemize}
    \item \textbf{Lower learning rates ($3\times10^{-5}$, $3\times10^{-4}$).} DroPE works effectively without QKNorm. At the lowest learning rate ($3\times10^{-5}$), the model without QK Norm achieves a slightly better final loss ($2.713$ vs. $3.102$). Together with the $3\times10^{-4}$ setting ($2.530$ vs.$2.555$), this indicates that QK Norm does not consistently improve performance in low-volatility regimes and is not the source of our gains.
    \item \textbf{High learning rate ($10^{-3}$).} At the highest learning rate, the model without QKNorm becomes unstable (loss spikes, gradient explosions), leading to poor convergence (final loss 6.334). In contrast, adding QKNorm stabilizes training and allows us to leverage the higher learning rate to achieve the best overall performance (final loss 2.496).
\end{itemize}

Figure~\ref{fig:qk_norm_ablation} shows the corresponding training curves with and without QK Norm, highlighting the presence of loss spikes at higher learning rates, in line with observations reported in~\cite{olmo2}. These results empirically demonstrate that the primary role of QK Norm is to act as a stabilizer that enables the use of a more aggressive, compute-efficient learning rate. Importantly, DroPE can still be applied without QK Norm by using a moderate learning rate (e.g., ($3\times10^{-4}$), which is our default setting for all experiments except the longer SmolLM-360M recalibration phases.

\begin{table}[H]
\centering
\caption{Ablation study on SmolLM-360M recalibration with and without QK Norm across different learning rates.}
\label{tab:qk-ablation}
\begin{tabular}{lccc}
\toprule
\textbf{Learning Rate} & \textbf{With QK Norm} & \textbf{Without QK Norm} & \textbf{Status} \\
\midrule
$10^{-3}$ (High) & $2.496$ & $6.334$ & Unstable without Norm \\
$3\times10^{-4}$ (Mid) & $2.555$ & $2.530$ & Stable / Comparable \\
$3\times10^{-5}$ (Low) & $3.102$ & $2.713$ & Stable / Comparable \\
\bottomrule
\end{tabular}
\end{table}

\begin{figure}[H]
    \centering
    \includegraphics[width=0.8\linewidth]{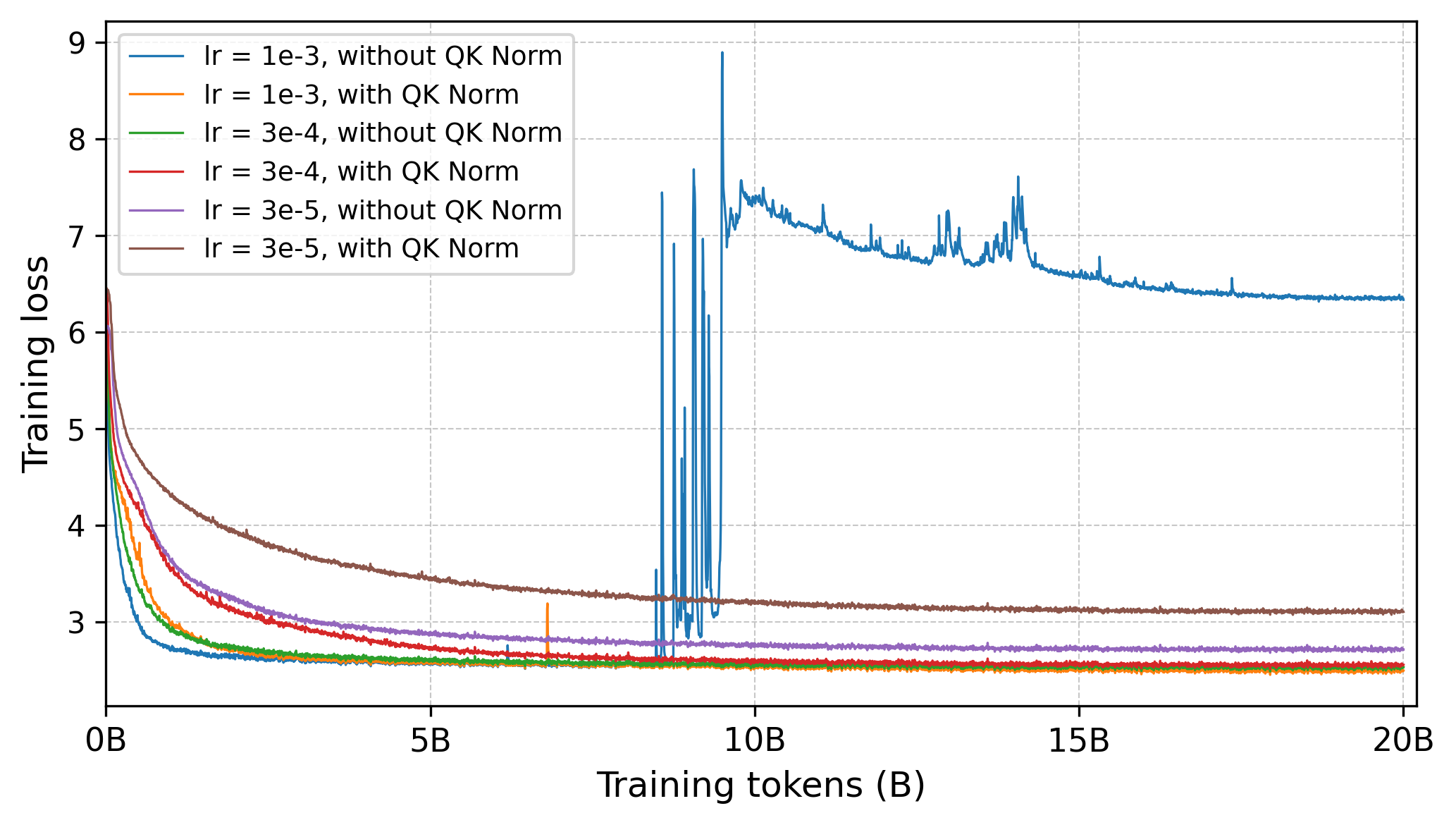}
    \caption{QKNorm allows for recalibration at a higher learning rate.}
    \label{fig:qk_norm_ablation}
\end{figure}

\end{document}